\documentclass[lettersize,journal]{IEEEtran}
\usepackage[caption=false,font=normalsize,labelfont=sf,textfont=sf]{subfig}
\usepackage{multicol}
\usepackage[bookmarks=true]{hyperref}

\usepackage{bm}
\usepackage{amssymb}
\usepackage{dsfont}
\usepackage{amsmath}

\usepackage{algorithm}
\usepackage{algorithmicx}
\usepackage{algpseudocode}

\usepackage{amsthm}
\usepackage{xcolor}
\usepackage{graphicx}
\usepackage{url}
\usepackage{svg}
\usepackage{nicefrac}
\usepackage{balance}
\usepackage{mathtools}
\usepackage{booktabs}
\usepackage{siunitx}
\usepackage{multirow} 
\usepackage{pgfplots}
\usepackage{pgfplotstable}
\usepackage{tikz}
\pgfplotsset{compat=newest}
\usepgfplotslibrary{groupplots}
\usepgfplotslibrary{units}
\usepgfplotslibrary{statistics}

\definecolor{customred}{RGB}{220,33,77}
\definecolor{lightred}{RGB}{255,150,150}
\definecolor{customblue}{RGB}{0,100,222}
\definecolor{lightblue}{RGB}{150,200,255}
\definecolor{customgray}{RGB}{120,120,120}
\definecolor{lightgray}{RGB}{220,220,220}

\newtheoremstyle{exampstyle}
  {3pt} % Space above
  {3pt} % Space below
  {\itshape} % Body font
  {} % Indent amount
  {\bfseries} % Theorem head font
  {.} % Punctuation after theorem head
  {.5em} % Space after theorem head
  {} % Theorem head spec (can be left empty, meaning `normal')

\theoremstyle{exampstyle} 
\newtheorem{definition}{Definition}
\newtheorem{lemma}{Lemma}
\newtheorem{theorem}{Theorem}
\newtheorem{remark}{Remark}
\newtheorem{assumption}{Assumption}
\newtheorem{problem}{Problem}
\newtheorem{proposition}{Proposition}
\newtheorem{corollary}{Corollary}
\newtheorem{example}{Example}
\newtheorem{example*}{Example*}

\theoremstyle{plain}

\newcommand{\RNum}[1]{\uppercase\expandafter{\romannumeral #1\relax}}
\newcommand{\ofx}{\left(\bm{x}\right)}

\newcommand{\ofxtau}{\left(\bm{x}_{\tau}\right)}
\newcommand{\xb}{\bm{x}}
\newcommand{\ub}{\bm{u}}
\newcommand{\zb}{\bm{z}}
\newcommand{\bb}{\bm{b}}

\algtext*{EndFor}
\algtext*{EndIf}

\begin{document}

\title{
Safety-critical Control Under Partial Observability: \\Reach-Avoid POMDP meets Belief Space Control}
% \title{Belief Control Lyapunov and Barrier Functions: A Layered Architecture for Safe Robot Decision Making under Uncertainty}

\author{Matti Vahs, Joris Verhagen and Jana Tumova
\thanks{This work was partially supported by the Wallenberg AI, Autonomous
		Systems and Software Program (WASP) funded by the Knut and Alice
		Wallenberg Foundation. This research has been carried out as part of the Vinnova Competence Center for Trustworthy Edge Computing Systems and Applications at KTH Royal Institute of Technology.
  }
	\thanks{The authors are with the Division of Robotics, Perception and Learning, KTH Royal Institute of Technology, Stockholm, Sweden and also affiliated with Digital Futures. Mail addresses: {\{\tt\small vahs, jorisv, tumova\}}
		{\tt\small @kth.se}}
}



\maketitle

\begin{abstract}
% POMDPs provide a principled framework for robot decision-making under partial observability, but in reach-avoid settings they require three competing behaviors with conflicting time-scale requirements: goal reaching, safety, and information gathering; the latter because neither safety nor goal reaching can be certified under high uncertainty. Existing online solvers typically combine these objectives within a single planning process, which can lead to conflicts in the objectives.
% We propose a layered, certificate-based control architecture for reach-avoid POMDPs that operates directly in belief space. We decouple control into three modules: a state-based reference controller for goal reaching, a belief Control Lyapunov Function (BCLF) for information gathering, and a belief Control Barrier Function (BCBF) for risk-aware safety. This separation allows controllers to operate at different frequencies with each module providing their own probabilistic guarantees.
% We formulate information gathering as a Lyapunov convergence problem in belief space and show how stochastic and finite-time BCLFs can be learned via reinforcement learning. For safety, we extend existing belief CBFs with conformal prediction to obtain finite-horizon probabilistic safety guarantees. Experiments in simulation and on a space-robotics platform demonstrate real-time performance and improved safety and task success compared to state-of-the-art constrained POMDP solvers.
Partially Observable Markov Decision Processes (POMDPs) provide a principled framework for robot decision-making under uncertainty. Solving reach-avoid POMDPs, however, requires coordinating three distinct behaviors: goal reaching, safety, and active information gathering to reduce uncertainty. Existing online POMDP solvers attempt to address all three within a single belief tree search, but this unified approach struggles with the conflicting time scales inherent to these objectives. We propose a layered, certificate-based control architecture that operates directly in belief space, decoupling goal reaching, information gathering, and safety into modular components. We introduce Belief Control Lyapunov Functions (BCLFs) that formalize information gathering as a Lyapunov convergence problem in belief space, and show how they can be learned via reinforcement learning. For safety, we develop Belief Control Barrier Functions (BCBFs) that leverage conformal prediction to provide probabilistic safety guarantees over finite horizons. The resulting control synthesis reduces to lightweight quadratic programs solvable in real time, even for non-Gaussian belief representations with dimension $>10^4$. Experiments in simulation and on a space-robotics platform\footnote{Video available at \href{https://youtu.be/0uxGLdUVQ10}{https://youtu.be/0uxGLdUVQ10}} demonstrate real-time performance and improved safety and task success compared to state-of-the-art constrained POMDP solvers.
\end{abstract}

\section{Introduction}

% Paragraph 1: Uncertainty is everywhere in robotics (unchanged, it's fine)
Uncertainty is a defining challenge in real-world robotic systems. Robots operate with noisy sensors, partial observability, and imperfect models of their own dynamics and environments. As a result, the true system state is rarely known exactly and must instead be inferred from noisy observations. 
Deploying robots under such uncertainties requires decision-making that accounts for incomplete state information and actively reasons about how uncertainty evolves over time.

% Paragraph 2: POMDPs and reach-avoid — TWO objectives, not three yet
A natural framework for reasoning under such uncertainty is the Partially Observable Markov Decision Process (POMDP)\cite{kurniawati2022partially}. A POMDP maintains a belief, i.e.\ a probability distribution over states, and seeks a policy that maps this belief directly to control actions. In this work, we focus on reach-avoid POMDPs, in which the robot must reach a desired goal set while avoiding unsafe states with high probability. Fig.\ref{fig:FirstPage} illustrates such a scenario: starting from an uncertain position (any blue particle), the robot must navigate to the goal region (green) without entering unsafe areas (red).

\begin{figure}
    \centering
    \includegraphics[width=\linewidth]{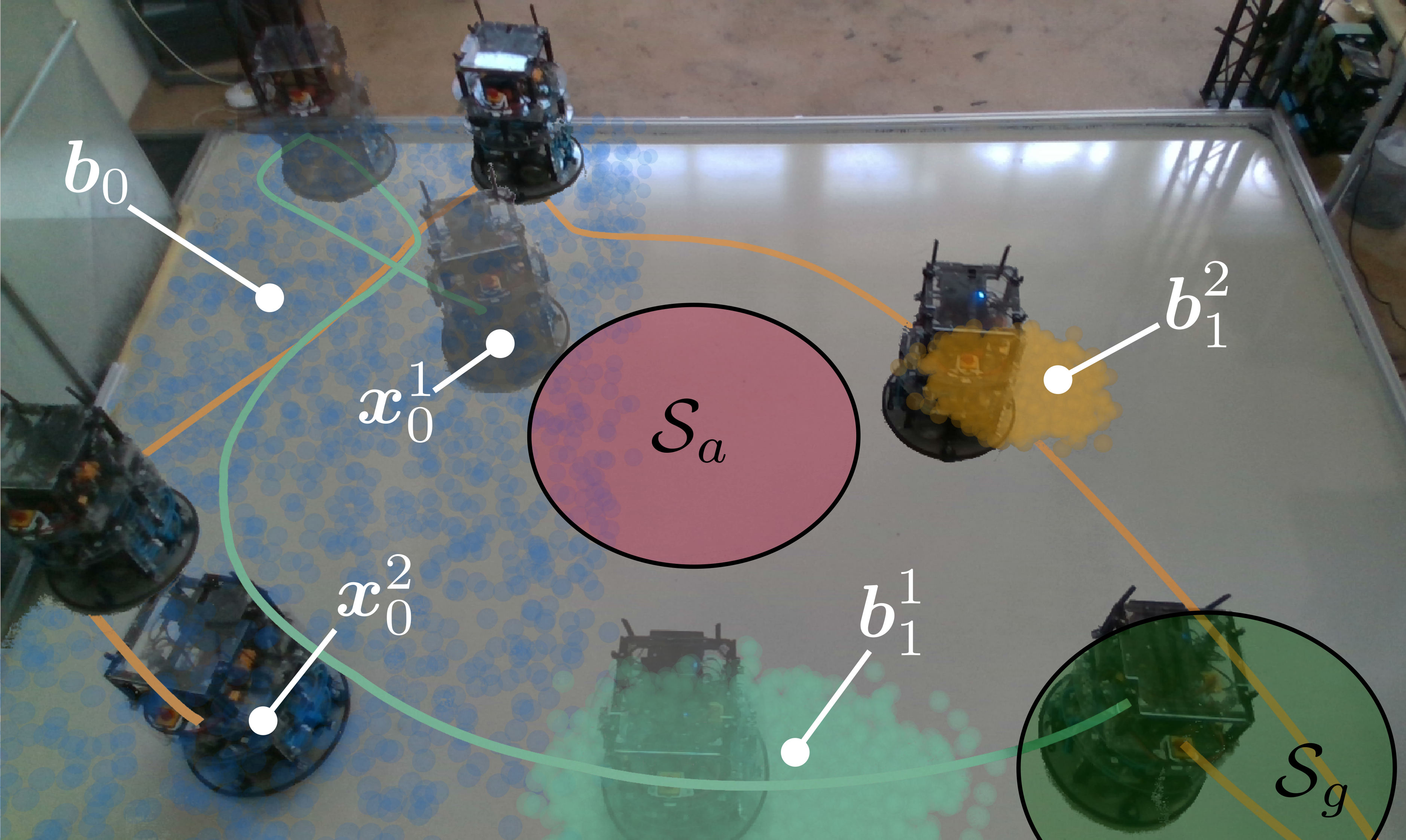}
    \vspace{-0.6cm}
    \caption{Illustration of our hardware experiments on a space-robotics platform. The robot starts from an unknown initial position that is located somewhere in the initial belief $\bm{b}_0$ and has to navigate to the goal region $\mathcal{S}_g$ without entering the avoid region $\mathcal{S}_a$. The robot can localize itself in a map by detecting impacts with the walls. Two runs with different initial conditions are shown.}
    \vspace{-0.3cm}
    \label{fig:FirstPage}
\end{figure}

% Paragraph 3: The problem — uncertainty prevents us from achieving either objective.
% IG emerges as an enabling behavior, turning two objectives into three behaviors.
However, under partial observability, neither objective can be achieved reliably when the robot's uncertainty about its own state is too large: it can neither certify that it has reached the goal nor guarantee that it will remain safe. To resolve this, the robot must actively gather information to reduce its uncertainty to a level at which both objectives can be ensured. Information gathering is therefore not a third objective that we impose, but one that the problem itself demands, turning the two objectives of the reach-avoid problem into three distinct control behaviors that must be coordinated during execution.

% Paragraph 4: Existing solvers merge everything into one search — timescale conflict.
Existing state-of-the-art online POMDP solvers attempt to address all three behaviors within a single belief tree search, encoding performance, safety, and information gathering into one optimization problem~\cite{zhitnikov2023simplified, zhitnikov2024anytime, moss2024constrainedzero}. While this unified formulation is very general, it introduces fundamental practical challenges. The different behaviors required in reach-avoid POMDPs operate on conflicting time scales. Satisfying safety constraints demands high-frequency, reactive control to prevent constraint violations in continuous time, whereas performance objectives such as goal reaching and information gathering benefit from longer planning horizons and coarser temporal abstraction. Encoding these competing requirements into a single online planning process leads to scalability issues and makes it difficult to deploy such solvers reliably in safety-critical robotic systems. If the planning time step is chosen to be small, the robot can react quickly and remain safe but lacks foresight, leading to overly conservative or short-sighted behavior. Conversely, larger planning time steps enable long-horizon information gathering and goal-reaching behavior, but at the cost of reduced reactivity and weaker safety assurances.

% Paragraph 5: Layered architectures as an alternative — motivate the central question.
In contrast, classic robotics pipelines have long relied on layered control architectures to manage complex tasks~\cite{tzafestas2018mobile}. By decomposing an overall problem into smaller subproblems, such as high-level planning, motion planning, and low-level control, these architectures allow each component to operate at an appropriate frequency while maintaining modularity and robustness. A similar decomposition in belief space would allow safety to be enforced at high frequency while information gathering and goal reaching operate on longer time scales, directly addressing existing shortcomings. However, designing such architecture is challenging because the control-theoretic tools that enable layered designs in state space do not readily extend to belief spaces. This motivates the central question of this paper: \emph{can we design a layered control architecture directly in belief space for solving reach-avoid POMDPs, and what are the missing components to do so?}

% Paragraph 6: "I could use CLFs/CBFs but..." — identify what's missing in each layer.
Control theory offers principled tools for realizing the individual layers of such an architecture through control certificates.  Control Lyapunov Functions (CLFs)~\cite{artstein1983stabilization} and control barrier functions (CBFs)~\cite{ogren2006autonomous, ames2016control} provide formal guarantees for stability and safety. Recent work has begun to extend CBF-based safety filters to belief spaces, enabling risk-aware safety guarantees under partial observability~\cite{vahs2023belief, vahs2024risk}. However, there is no straightforward way to design certificate-based controllers for information gathering, which is essential for reducing uncertainty and enabling reliable goal reaching. In this work, we propose a control architecture that explicitly decouples goal reaching, safety, and information gathering, allowing robots to reach their goals safely while providing probabilistic guarantees associated with each certificate.

% Paragraph 7: What we do — we don't just assemble layers, we advance each one.
This modular design has several advantages. An information-gathering controller in belief space can be reused across different reach-avoid objectives without redesign. Goal-reaching controllers can be designed using standard state-based techniques, and existing belief-space CBF methods can be leveraged for safety enforcement. At the same time, our approach overcomes a key limitation of prior belief-space safety filters by moving beyond pointwise-in-time guarantees and enabling probabilistic safety guarantees over finite horizons. Together, these ideas lead to a scalable and flexible framework for solving reach-avoid POMDPs and deploying them on hardware.

\subsection{Contributions}
We propose Belief Control Lyapunov and Barrier Functions (BCLF and BCBF, respectively) in a layered control architecture for solving reach-avoid POMDPs through belief-space control. Specifically, we make the following contributions:
\begin{enumerate}
    \item We formalize information gathering as the problem of finding a valid CLF in a non-Gaussian belief space, framing it as an enabling behavior for both goal reaching and safety;
    \item We present a method for learning belief CLFs through reinforcement learning and establish theoretical conditions under which RL value functions constitute valid stochastic and finite-time CLFs;
    \item We develop a risk-aware safety filter as a belief CBF that provides probabilistic safety guarantees over a finite horizon, moving beyond pointwise-in-time assurances;
    \item We show that the resulting modular architecture, combining state-based reference controllers, BCLFs, and BCBFs, improves both safety and task performance over state-of-the-art constrained POMDP solvers;
    \item We validate our approach on a space robotics hardware platform in multiple challenging scenarios with non-Gaussian beliefs, synthesizing control inputs in real time for belief states with dimension $>10^4$.
\end{enumerate}
This work expands on initial results on Belief Control Barrier Functions~\cite{vahs2023belief, vahs2024risk}, which designed risk-aware safety filters in Gaussian and non-Gaussian belief spaces. We provide both practical and theoretical advancements of belief space safety filters suitable for reach-avoid POMDPs. Furthermore, the design of Belief Control Lyapunov Functions for information gathering has not been considered before.

 \section{Related Work} 

\subsection{POMDP Solvers}
Offline POMDP solvers compute a policy prior to execution by approximating the optimal value function over the belief space. Classical point-based methods \cite{porta2006point, spaan2005perseus, smith2012heuristic, smith2012point, kurniawati2008sarsop} are effective for small, discrete problems but require significant offline computation and memory. They do not scale well to large or continuous state and observation spaces, limiting their applicability in robotics~\cite{kurniawati2022partially}.

Online POMDP solvers address these limitations by interleaving planning and execution. Rather than computing a global policy, they build a belief tree from the current belief using online tree search and Monte Carlo sampling, selecting actions that maximize expected reward over a finite, receding horizon. Recent online solvers leverage particle-based beliefs and generative models to handle continuous state and observation spaces \cite{garg2019despot, hoerger2021line, lim2021sparse, mern2021bayesian, sunberg2018online, barenboim2023online}.

Despite their improved scalability, standard online POMDP solvers are primarily designed for maximizing expected reward. This formulation is not straightforward to apply in safety-critical settings, where violations may be rare but catastrophic. To address this, constrained POMDP (cPOMDP) formulations have been proposed, in which reward maximization is subject to safety constraints. Some works enforce constraints in expectation \cite{lee2018monte, jamgochian2023online, stocco2024addressing}, while others consider probabilistic or chance-constrained formulations that explicitly bound the probability of constraint violation \cite{zhitnikov2023simplified, zhitnikov2024anytime, moss2024constrainedzero}.

A fundamental challenge shared by these approaches is that online POMDP solvers attempt to encode both performance objectives and safety requirements within a single tree search. These objectives can be conflicting: safety often requires high-frequency control, whereas performance benefits from longer planning horizons and coarser temporal abstraction. To mitigate this issue, recent work has explored hierarchical POMDP solvers, where higher-level planning is performed over abstract actions or skills \cite{jamgochian2024constrained}. 
%However, existing hierarchical approaches typically do not execute controllers at different time scales, and the construction or learning of the skill hierarchy is often assumed rather than explicitly addressed.

\subsection{Belief Space Control}
Although POMDP solvers operate on beliefs, they primarily rely on belief tree search. In contrast, belief space planning and control strategies explicitly model belief dynamics and synthesize plans or policies directly in the belief space. This line of work originates in belief space planning (BSP), which formulates trajectory optimization problems over belief states.

Early BSP methods adopt Gaussian belief representations, propagating uncertainty through linearized dynamics and Gaussian filters. Foundational work formulates trajectory optimization directly in this belief space, allowing robots to balance task performance with information gathering and, in some cases, incorporate probabilistic safety constraints \cite{Platt2010BeliefSP, van2012motion, vitus2011closed, vahs2023risk}. Although computationally efficient, these approaches are constrained by relying on Gaussian assumptions.

To address more complex beliefs, non-Gaussian BSP methods have been proposed. In particular, particle-based belief representations are used to capture multi-modal uncertainty, and planning is formulated as a reachability or trajectory optimization problem in belief space \cite{platt2012non, platt2016efficient}. These methods seek information-gathering actions by ruling out competing hypotheses, but still primarily produce offline belief space plans rather than feedback policies.

A key limitation of BSP is that it generates a belief space plan that is not explicitly adapted online as new observations arrive. This has motivated online belief space control strategies, which compute control inputs directly as a function of the current belief. A subset of this literature focuses on controlling an external belief, such as a map, a field, or moving targets in the environment using information-theoretic objectives to reduce uncertainty \cite{atanasov2014information, cai2021non, yang2022learning, acp2024}. Related work learns policies directly from Gaussian belief states for tasks such as active target tracking \cite{yang2022learning}.

Bridging POMDPs and control theory, \cite{ahmadi2020control} propose a formalism that treats POMDPs as hybrid systems and defines Lyapunov- and barrier-like certificates in belief space to characterize reachable belief sets and verify safety properties. While this work establishes an important theoretical connection, it focuses on verification rather than control synthesis. In contrast, our work aims to directly synthesize control laws and certificates in belief space.

In this direction, belief space safety filters based on control barrier functions (CBFs) have been proposed. For Gaussian belief states, belief CBFs enforce probabilistic safety constraints pointwise-in-time using EKF-based belief dynamics \cite{vahs2023belief}, with extensions to state-dependent noise in~\cite{walia2025belief}. These ideas have been further generalized to non-Gaussian belief spaces using particle filters \cite{vahs2024risk, han2025risk}. A major advantage of these methods is scalability, as control synthesis reduces to solving a single quadratic program. This can be done efficiently in real-time even for belief representations with thousands of particles~\cite{vahs2024risk}. However, existing work focuses on risk-aware safety guarantees and does not address the synthesis of other belief space certificates for performance objectives.

Finally, early information-theoretic control approaches define Lyapunov-like objectives over belief uncertainty, such as mutual information, to guide exploration in POMDPs \cite{hoffmann2009mobile, hoffmann2006mutual}. These methods optimize only one-step uncertainty reduction, which can result in greedy behavior and convergence to local minima.

\subsection{Learning Control Certificates}
Since our architecture relies on constructing CLFs and CBFs in belief space, we review methods for learning such certificates. CLFs characterize stability by certifying the existence of control inputs that drive the system toward a desired equilibrium \cite{artstein1983stabilization}, while CBFs encode safety as forward invariance of a safe set \cite{ogren2006autonomous, ames2016control}. Despite their strong theoretical foundations, constructing suitable certificates for general nonlinear and stochastic systems remains a largely open problem.

Data-driven approaches have emerged that learn continuous-time certificates using neural networks, typically by uniformly sampling the state space and enforcing Lyapunov or barrier conditions at sampled points \cite{dawson2022learning, dawson2023safe, dawson2022safe}. Similar ideas have been applied to learning barrier functions from observations \cite{long2021learning, xiao2023barriernet}. Extensions to stochastic systems include distributionally robust Lyapunov functions \cite{long2023distributionally, long2024distributionally}, stochastic CLFs \cite{zhang2022neural, neustroev2025neural}, and stochastic CBFs \cite{tayal2024learning, zhang2025stochastic}. While these methods alleviate the design of control certificates, they scale poorly to high dimensions due to their reliance on dense uniform sampling over the state space.

A complementary line of work draws connections between control certificates and value functions. In continuous time, \cite{so2024train} establish a relationship between CBFs and value functions, while in discrete time, \cite{tan2024safe} show that an RL value function can directly serve as a CBF under a suitable reward structure. Earlier work \cite{berkenkamp2017safe} argues that value functions can act as CLFs for deterministic systems, though precise conditions are not fully specified. More recently, several works have linked RL with Hamilton–Jacobi reachability to learn safety value functions \cite{nakamura2025generalizing, agrawal2025anysafe} and latent-space CBFs \cite{nakamura2025trainlatentcontrolbarrier}. Finally, \cite{du2023reinforcement} propose a Lyapunov–barrier actor–critic method that couples policy learning with certificate satisfaction.

Overall, existing learning-based certificate methods perform well in low-dimensional deterministic settings but face challenges in high-dimensional belief spaces with stochastic dynamics. This motivates our approach of learning certificates via reinforcement learning in belief space, for which we provide conditions under which a value function exhibits the same theoretical properties as stochastic CLFs.

\section{Preliminaries}
\subsection{Stochastic Control Barrier Functions}
\label{sec:SCBFs}
Consider the continuous-time stochastic dynamical system in control affine form
\begin{align}
    \mathrm{d}\bm{x} &= \left(\bm{f}\left(\bm{x}\right) + \bm{g}\left(\bm{x}\right)\bm{u}\right) \mathrm{d}t + \bm{\sigma}\ofx ~\mathrm{d}\bm{W}\label{eq:SDE}
\end{align}
where $\bm{x} \in \mathcal{X} \subseteq \mathbb{R}^{n_x}$ is the state, $\bm{u} \in \mathcal{U} \subseteq \mathbb{R}^m$ is the control input, $\bm{W} \in \mathbb{R}^q$ is $q\text{-dimensional}$ Brownian motion and $\bm{\sigma}\ofx$ is a diffusion term. Next, we present the necessary theory to ensure set invariance under stochastic dynamics. Let
\begin{equation}
\begin{aligned}
  \label{eq:safeset_state}
    \mathcal{C}_x &= \left\{\bm{x} \in \mathcal{X} \mid h_x\left(\bm{x}\right) \geq 0\right\}
 \end{aligned}
\end{equation} 
be a safe set in which we want the state to remain in.  
Stochastic CBFs have been proposed in \cite{clark2021control, so2023almost} which render the safe set $\mathcal{C}_x$ forward invariant.
% assuming that the function $h_x\ofx:\mathcal{X}\mapsto \mathbb{R}$ is twice continuously differentiable. 
\begin{definition}
\label{def:RCBF}
        Given a safe set $\mathcal{C}_x$ defined by \eqref{eq:safeset_state}, the function $B_x\ofx$ is a stochastic reciprocal control barrier function (RCBF) if 
        \begin{enumerate}
            \item there exist class-$\kappa$ functions\footnote{A function $\alpha:\mathbb{R}\mapsto\mathbb{R}$ is class-$\kappa$ if it is strictly increasing and $\alpha(0)=0$} $\alpha_1$ and $\alpha_2$ such that 
            \begin{align}
                \frac{1}{\alpha_1\left(h_x\left(\bm{x}\right)\right)} \leq B_x\ofx \leq \frac{1}{\alpha_2\left(h_x\left(\bm{x}\right)\right)},
            \end{align}
            \item there exists a class-$\kappa$ function $\alpha_3$ and a $\bm{u} \in \mathcal{U}$ such that
            \begin{align*}
            \frac{\partial B_x}{\partial \bm{x}} \left(\bm{f}\ofx + \bm{g} \ofx \bm{u}\right) + \frac{1}{2} \mathrm{tr}\left[\bm{\sigma}^T \frac{\partial^2 B_x}{\partial \bm{x}^2} \bm{\sigma}\right]\\
            \leq \alpha_3(h_x\ofx).
            \end{align*}
        \end{enumerate}
    \end{definition}
    \begin{theorem}[Thm. 2 in \cite{clark2021control}]
    \label{thm:SCBF}
    If $B_x$ is a RCBF as in Def.~\ref{def:RCBF} for the continuous-time system~\eqref{eq:SDE} and at each time $t$, $\ub(t)$ satisfies Def.~\ref{def:RCBF}, then $\mathrm{Pr}\left[\bm{x}(t) \in \mathcal{C}_x, \hspace{0.2cm} \forall t\geq t_0\right]=1$, provided that $\xb(t_0) \in \mathcal{C}_x$.
    \end{theorem}

\subsection{Stochastic Lyapunov Theory}
Consider the discrete-time stochastic dynamical system
\begin{align}
    \xb_{k+1} = \bm{F}(\xb_k, \ub_k, \bm{w}_k)\label{eq:discrete_time_sys}
\end{align}
where $\bm{w} \in \mathbb{R}_q$ is a random variable distributed according to $\bm{w}_k\sim p(\bm{w})$. Subsequently, we introduce the necessary theory to guarantee convergence to a goal set $\mathcal{G}\subset\mathcal{X}$ in state space.
\begin{definition}
\label{def:SCLF}
    A function $W :\mathcal{X}\mapsto\mathbb{R}$ is a stochastic Control Lyapunov Function (SCLF) for the dicrete-time system~\eqref{eq:discrete_time_sys} and the goal set $\mathcal{G} \subset\mathcal{X}$ if it satisfies
    \begin{align*}
        &W(\xb) \leq 0 \hspace{0.5cm} \forall \xb \in \mathcal{G}\\
        &W(\xb) > 0 \hspace{0.5cm} \forall \xb \in \mathcal{X}\setminus \mathcal{G}\\
        &\exists\ub \in \mathcal{U}:~\mathbb{E}\left\{W(\bm{F}(\xb_{k}, \ub_k, \bm{w}_k))\right\} \leq c W(\bm{x}_k)\hspace{0.5cm} \forall \xb \notin \mathcal{G},
    \end{align*}
    for some $c \in (0, 1)$.
\end{definition}
\begin{theorem}[Thm. 2.3 in \cite{haddad2022lyapunov}]
\label{thm:SCLF}
    If $W\ofx$ is a SCLF as in Def.~\ref{def:SCLF} and the control input satisfies Def. \ref{def:SCLF}, then
    \begin{align}
        \underset{k \rightarrow \infty}{\mathrm{lim}} W(\xb_k) = 0
    \end{align}
    almost surely which implies that the state asymptotically converges to the goal set almost surely.
\end{theorem}

\begin{definition}
\label{def:FSCLF}
    A function $W :\mathcal{X}\mapsto\mathbb{R}$ is a finite-time stochastic Control Lyapunov Function (FSCLF) for the discrete-time system~\eqref{eq:discrete_time_sys} and the goal set $\mathcal{G} \subset\mathcal{X}$ if it satisfies
    \begin{align}
        &W(\xb) \leq 0 \hspace{0.5cm} \forall \xb \in \mathcal{G}\\
        &W(\xb) > 0 \hspace{0.5cm} \forall \xb \in \mathcal{X}\setminus \mathcal{G}\\
        &\exists\ub \in \mathcal{U}:~\Delta W \leq -\min(W\ofx, \eta)~ \forall \xb \in \mathcal{X}\notin \mathcal{G},\label{eq:SCLF2} 
    \end{align}
    with $\Delta W = \mathbb{E}\left\{W(\bm{F}(\xb_{k}, \ub_k, \bm{w}_k))\right\} - W\ofx$ and for some $\eta > 0$.
\end{definition}
\begin{theorem}[Thm. 4.2 in \cite{lee2022finite}]
\label{thm:FSCLF}
    If $W\ofx$ is a FSCLF as in Def.~\ref{def:FSCLF} and the control input satisfies Def. \ref{def:FSCLF}
    then there exists an almost surely finite stochastic settling time $K(\xb_0)$ such that the state reaches $\mathcal{G}$ in finite time, i.e.
    \begin{align}
        \mathrm{Pr}\left[W(\xb_{K(\xb_0)})=0\right]=1
    \end{align}
    with $\mathbb{E}\{K(\xb_0)\} \leq \lceil \nicefrac{W(\xb_0)}{\eta}\rceil$.
\end{theorem}

\subsection{Conformal Prediction}
\label{sec:CP}
Conformal Prediction (CP) is a lightweight statistical tool for uncertainty quantification that can enable practical safety guarantees for autonomous systems \cite{shafer2008tutorial, lindemann2024formal}. 
\begin{lemma}
    Let $Z, Z^{(1)},\dots, Z^{(k)}$ be k + 1 independent and identically distributed (i.i.d.) real-valued random variables. Let $Z^{(1)},\dots, Z^{(k)}$ be sorted in non-decreasing order and define $Z^{k+1} := \infty$. For $\delta \in (0, 1)$, it holds that 
    \begin{align*}
        \mathrm{Pr}\left[Z \leq \bar{Z}\right] \geq 1- \delta
    \end{align*}
    where $\bar{Z} = Z^{(r)}$ with $ r = \lceil(k+1)(1-\delta)\rceil$
    and where $\lceil\cdot\rceil$ is the ceiling function.
    \label{lemma:CP}
\end{lemma}
The minimum number of required samples to provide a probabilistic guarantee is given by 
$k \geq \frac{1 - \delta}{\delta}$. The variable $Z$ is usually referred to as the \emph{non-conformity score}.

\section{Problem Setting}
This paper is concerned with the problem of making a robot reach a desired goal set in state space while avoiding a set of unsafe states when the initial state is uncertain and can only be estimated based on noisy partial observations. We consider continuous-time systems whose nonlinear dynamics are described by stochastic differential equations (SDEs) in control-affine form as in Eq.~\eqref{eq:SDE}. 
% We assume that the diffusion term $\bm{\sigma}\ofx$ is a globally Lipschitz non-degenerate diagonal matrix and the drift term $\left(\bm{f}\left(\bm{x}\right) + \bm{g}\left(\bm{x}\right)\bm{u}\right)$ is a locally Lipschitz function that can have jumps. 
% These assumptions ensure that Eq.~\eqref{eq:SDE} has a unique global strong solution CITE.

At discrete time steps $t_k$, the robot makes an observation $\zb$ that is a nonlinear stochastic function of the state
\begin{align}
    \bm{z}_{t_k} &= \bm{\ell}\left(\bm{x}_{t_k}, \bm{v}_{t_k}\right)\label{eq:observation}
\end{align}
with measurement noise $\bm{v}_{t_k} \sim p(\bm{v})$. Measurements only occur at discrete timesteps $t_1,\ldots, t_k$ due to limited update rates of sensors.
Thus, the exact state at time $t$ is unknown, but a Bayesian posterior $p\left(\bm{x}_t \mid \bm{z}_{t_1:t_k}\right)$, for $t_k \leq t < t_{k+1}$  describes the probability distribution over states conditioned on past measurements.

In its most general form, the evolution of the posterior distribution can be seen as a stochastic hybrid system that has two different evolutions: 

1) In periods where no measurements are available, the belief evolves according to the Fokker-Planck-Kolmogorov partial differential equation~\cite{jazwinski2007stochastic}
\begin{multline}
    \frac{\partial p(\xb_t\mid\zb_{t_1:t_k})}{\partial t} = - \sum_{i=1}^{n_x} \frac{\partial}{\partial x_i}\left(p(\xb_t\mid\zb_{t_1:t_k})(\cdot)_i\right) +\\
     \frac{1}{2}\sum_{i,j=1}^{n_x} \frac{\partial^2}{\partial x_i x_j} \left(p(\xb_t\mid\zb_{t_1:t_k}) \left(\bm{\sigma}\ofx \bm{\sigma}\ofx^T\right)_{ij}\right)\label{eq:FokkerPlanck}
\end{multline}
where $(\cdot) = (\bm{f} \ofx + \bm{g} \ofx \ub)$ and the subscripts $i$ and $j$ refer to $i$th and $j$th element, respectively. 

2) Whenever a new measurement is available, the state distribution is updated using Bayes' rule
\begin{align}
    p(\xb_t\mid\zb_{t_1:t_{k+1}}) = \frac{p(\zb_{t_{k+1}}\mid\xb_{t})p(\xb_{t}\mid\zb_{t_1:t_{k}})}{p(\zb_{t_{k+1}}\mid\zb_{t_1:t_{k}})}.\label{eq:Bayes}
\end{align}
However, in general, neither the continuous evolution nor the discrete update admit closed-form solutions for nonlinear system which is why a variety of approximations have been proposed \cite{sarkka2006recursive, mallick2012continuous}.

The approach taken in this work is to model the reach-avoid problem in belief space and to use finite sample guarantees on the resulting belief states. Thus, we seek to find the sets in belief space such that convergence to that set implies the robot reaching its goal set, denoted by $\mathcal{S}_g \subset \mathcal{X}$, with desired probability while ensuring that the state does not enter a set of unsafe states denoted by $\mathcal{S}_a\subset \mathcal{X}$. Formally, we define the problem hereafter.
\begin{problem}[Reach-Avoid POMDP]
    \label{prob1}
    Given a nonlinear stochastic system in Eq.~\eqref{eq:SDE} and stochastic measurements according to Eq.~\eqref{eq:observation} with a goal set $\mathcal{S}_g \subset \mathcal{X}$ and avoid set $\mathcal{S}_a\subset \mathcal{X}$, find control inputs $\ub \in \mathcal{U}$ such that for all states starting from a given initial distribution $\xb_0 \sim p(\xb_0)$ the resulting trajectory satisfies
    \begin{align}
        &\textbf{reach}: &&\exists t\geq 0 \text{ s.t }\mathrm{Pr}[\xb_t \in \mathcal{S}_g] \geq 1 - \delta_g\\
        &\textbf{avoid}: &&\mathrm{Pr}[\xb_t \notin \mathcal{S}_a, \forall t\in[0,T]] \geq 1 - \delta_a
    \end{align}
    with user defined probability thresholds $\delta_g$ and $\delta_a$.
\end{problem}
Here, we want to emphasize that our problem setting is a subset of a general POMDP since both problems are concerned with the control of partially observable systems. In the more general POMDP setting, the objective is to maximize the expected cumulative discounted reward. Specifically, our problem can be seen as a special case of a constrained POMDP with set-based reach and avoid objectives. Additionally, POMDPs are typically defined in discrete-time while we consider the continuous-discrete setting commonly arising in robotics and control problems.
\begin{figure}
    \centering
    \includegraphics[width=0.5\textwidth]{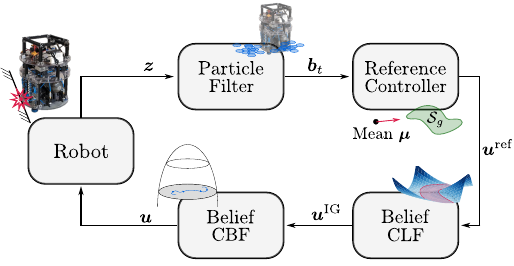}
    \vspace{-0.4cm}
    \caption{Illustration of our proposed control architecture on the example of a free floating robot platform that uses impact detections as measurements. The robot uses measurements $\zb$ to update a particle filter belief $\bb$ at time $t$. Based on the current belief, the reference controller obtains an input based on the mean state $\bm{\mu}$, the belief CLF serves as information gathering controller and the belief CBF minimally corrects all unsafe control inputs.}
    \vspace{-0.4cm}
    \label{fig:ControlStructure}
\end{figure}

\subsection{Overview}
To address Problem 1, we formulate a belief space control problem which is composed of different control strategies that are combined into one control architecture depicted in Fig.~\ref{fig:ControlStructure}. First, we introduce a continuous-discrete particle filter as approximated belief dynamics model of the general posterior evolution in Eq.~\eqref{eq:FokkerPlanck} and~\eqref{eq:Bayes}. Then, we decompose the overall control strategy into three modules that address different subproblems of the reach-avoid POMDP. We argue that, e.g., information gathering only needs to happen at the frequency of measurements while the safety filter needs high frequency updates to ensure that the continuous-time system does not enter the avoid set. Specifically, we make the following design choices:
\begin{description}
    \item[1)]A state-based \textbf{reference} controller is used which acts on the mean state of the particle filter belief and drives the mean state towards the goal region. 
    \item[2)] An \textbf{information gathering controller} (BCLF) is used to actively reduce uncertainty in the robot's belief in the direction of the state-based reference. If the IG controller localizes the state sufficiently well, the robot will reach its desired goal location.
    \item[3)]A \textbf{belief safety filter} (BCBF) that ensures that the probability of entering the avoid set is bounded during operation. This safety filter corrects all potentially unsafe controls obtained by the reference and IG controllers. 
\end{description}
In the control design of the IG controller and the safety filter, we specifically design certificate-based controllers, i.e. we show how information gathering can be interpreted as finding a valid Control Lyapunov Function and how risk-aware safety boils down to finding a Control Barrier Function in belief space.

\section{Belief Space Model}
The general Bayesian posterior $p(\xb_t \mid \zb_{t_1:t_k})$ evolves according to the Fokker--Planck PDE between measurements and a nonlinear Bayes update at measurement times (Eqs.~\eqref{eq:FokkerPlanck} \& \eqref{eq:Bayes}). For nonlinear, continuous-time systems, this belief evolution is intractable to compute or propagate in closed form, and its infinite-dimensional nature prohibits direct use in control synthesis. To obtain a tractable yet expressive approximation, we represent the belief with a continuous--discrete particle filter, which provides a sample-based approximation of the posterior while preserving non-Gaussian structure. Under this approximation, the belief dynamics become a high-dimensional stochastic hybrid system in which particles follow the continuous-time SDE~\eqref{eq:SDE} between measurements and undergo a discrete Bayesian update through weighting and resampling when new observations arrive. This representation enables explicit belief-space control design while maintaining the ability to capture multimodal and highly nonlinear uncertainties inherent in our setting.

PF techniques stem from the idea of approximating the Bayesian posterior by means of a finite set of $N$ weighted samples,  i.e. particles, $\{(\bm{x}_t^{(i)}, w_t^{(i)})\}_{i=1}^N$, where $w_t^{(i)} \in \mathbb{R}_{\geq 0}$. The PF is a non-parametric approach that approximates the true posterior as
 \begin{align*}
    p\left(\bm{x}_t\right) \approx \sum_{i=1}^N w_t^{(i)} \delta\left(\bm{x}_t - \bm{x}_t^{(i)}\right)
\end{align*}
where $\delta(\cdot)$ denotes the Dirac delta function. It is known that as $N \rightarrow \infty$, the approximated posterior converges to the true posterior \cite{984773}.
Thus, we define a belief state
\begin{align*}
    \bm{b}_t = \begin{bmatrix}\bm{x}_t^{(1)} & \hdots & \bm{x}_t^{(N)}\end{bmatrix}^T \in \mathbb{R}^{N \cdot n_x}.
\end{align*}
which describes the current PF belief at time $t$ as a collection of all particles. 
The continuous time dynamics of the belief state $\bm{b}$ during the period where no observations are available is obtained by propagating each particle through the nonlinear SDE in Eq.~\eqref{eq:SDE}, resulting in
\begin{align}
    \mathrm{d}\bm{b}_t 
    &= \left(\begin{bmatrix}
        \bm{f}\left(\bm{x}_t^{(1)}\right) \\
        \vdots\\
        \bm{f}\left(\bm{x}_t^{(N)}\right)  
    \end{bmatrix} + \begin{bmatrix}
        \bm{g}\left(\bm{x}_t^{(1)}\right) \\
        \vdots\\
        \bm{g}\left(\bm{x}_t^{(N)}\right)  
    \end{bmatrix}\bm{u}\right) \mathrm{d}t + \begin{bmatrix}
    \bm{\sigma} ~\mathrm{d}\bm{W}_t^{(1)}\\
    \vdots\\
    \bm{\sigma}~ \mathrm{d}\bm{W}_t^{(N)}
    \end{bmatrix}\nonumber\\
    &:= \left(\bm{f}_b\left(\bm{b}\right) + \bm{g}_b \left(\bm{b}\right) \bm{u}\right) \mathrm{d}t + \bm{\Sigma} ~ \mathrm{d}\tilde{\bm{W}}\label{eq:BeliefDynamics}
\end{align}
where $\bm{\Sigma} = \mathrm{BD}\left(\{\bm{\sigma}\}_{i=1}^N\right)$ is a block diagonal matrix and $\tilde{\bm{W}}$ is a Brownian motion of dimension $N \cdot q$. 
Note that, although the belief dynamics are of high dimension, the individual state particles are entirely decoupled which means that the SDE in Eq.~\eqref{eq:BeliefDynamics} can be solved efficiently using parallelization. Further, the belief dynamics stay control-affine which alleviates the control design under high-dimensional belief states.

At discrete times $t_k=t_k^+$ where observations are available, the PF belief is conditioned on the most recent measurement which leads to a discrete update of the belief state, i.e. $\bm{b}(t_k^+) = \bm{\Delta}(\bm{b}(t_k^-), \bm{z}_k)$ where $t_k^-$ is infinitesimally smaller than $t_k^+$. 
This discrete map $\bm{\Delta}(\cdot)$, however, cannot be obtained in closed form and is performed as a stochastic resampling step. Specifically, each particle $\xb_t^{(i)}$ is weighted and resampled according to its observation likelihood $w_t^{(i)} = p(\zb_t\mid \xb_t^{(i)})$. We perform low-variance resampling, which reduces particle impoverishment and yields a more stable approximation of the posterior distribution.
For a detailed description of this step, the reader is referred to \cite{thrun2005probabilistic}. 

Finally, combining the SDE in Eq.~\eqref{eq:BeliefDynamics} with the discrete-time measurement update leads to a hybrid system formulation
% \begin{align}
%     \mathcal{S} = 
%     \begin{cases}
%         \dot{\bm{b}} = \bm{f}_b\left(\bm{b}\right) + \bm{g}_b\left(\bm{b}\right) \bm{u}, & \forall t \in [t_{k-1}, t_k)\\
%         \bm{b}^+ = \bm{\Delta}\left(\bm{b}^-, \zb_k\right) & t = t_k,
%     \end{cases}\label{eq:hybridsys}
% \end{align}
\begin{equation}
\label{eq:hybridsys}
    \mathcal{S}=
    \begin{cases}
        \text{d}{\bb} = (\bm{f}_b(\bb) + \bm{g}_b(\bb)\ub)\text{d} t + \bm{\Sigma}~\text{d}\Tilde{\bm{W}}, \enskip \forall t\in[t_{k-1},t_k)\\
        \bm{b^+ = \Delta(b^-,\zb_{t_k})},\enskip t = t_k.
    \end{cases}
\end{equation}
describing the evolution of the robot's belief over time.

\begin{assumption}
    \label{ass:iid}
    Throughout this paper, we assume that the resampled particles after a measurement update consist of $N$ i.i.d. samples from the true belief.
\end{assumption}
While this assumption does not hold for a small number of particles, we know that the resampled particles become asymptotically i.i.d. from the true distribution as $N\rightarrow\infty$. We rely on this assumption to make statistical statements about safety and performance using Conformal Prediction. Additionally, different resampling schemes such as independent resampling, proposed in \cite{lamberti2017independent, lamberti2017semi}, can be used to make the resampling step independent.

\subsection{Uncertainty Quantification Of The Belief}
In order to design a controller that actively gathers information through measurements, we first need to be able to quantify what information means in this context. Various information theoretic measures have been proposed in the past and have been adopted for informative path planning in different applications. A widely used uncertainty measure for continuous probability distributions is the differential entropy
\begin{align}
    \mathcal{H} \ofx = - \int_{\mathcal{X}} p\ofx \mathrm{log}(p\ofx) \text{d}\xb
\end{align}
which quantifies the spread of a continuous random variable $\xb$ based on its probability density function. Differential entropy has been used in many planning and control problems under uncertainty, also in combination with particle filters \cite{hoffmann2009mobile}. However, when substituting the unweighted particle filter belief it becomes apparent that the entropy approximates as
\begin{align}
    \mathcal{H}\ofx \approx - \sum_{i=1}^N \frac{1}{N} \mathrm{log}\left(\frac{1}{N}\right) = \text{const.}
\end{align}
which does not capture any information during a continuous time period. Even if a weighted formulation of the particle filter is used, the discrete approximation fails to capture the spatial distribution of particles but only depends on the weights, i.e. the observation likelihood of particles \cite{brudermuller2024tactile}. Thus, using differential entropy as a uncertainty quantification is avoided in this work since it is poorly approximated in particle filter beliefs.

To that end, we introduce a new measure of uncertainty tailored to PF-based beliefs. The key idea builds on the assumption that the particle filter provides independent samples from the underlying belief distribution. To characterize how concentrated this belief is, we consider a ball of radius $\varepsilon$ around the mean state estimate whose size reflects our uncertainty: the smaller the ball, the more confidently the robot is localized. Conformal prediction applied to the particle samples enables us to construct such a ball with a guaranteed probability of containing the true state, thereby providing a distribution-free measure of uncertainty.

We define the nonconformity score
\begin{align}
    \rho\ofx = \lVert \xb - \mathbb{E}\{\xb\}\rVert_2\label{eq:localizability}
\end{align}
\begin{figure}
    \centering
    \includegraphics[width=0.49\textwidth]{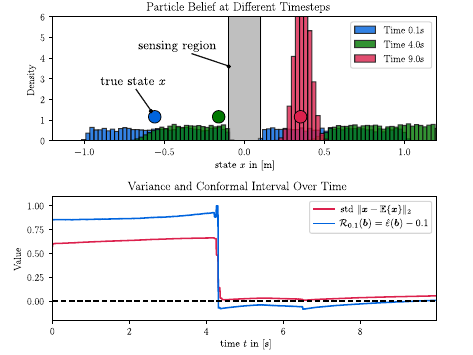}
    \vspace{-0.4cm}
    \caption{Illustration of an example with a one dimensional state space. In this environment, the state can be accurately localized in a sensing region $\{ x \in \mathbb{R}\mid |x| \leq 0.1\}$ around the origin. The true state evolves according to a continuous-time SDE. The bottom plot shows the proposed uncertainty quantification as well as the belief empirical standard deviation over time.}
    \vspace{-0.4cm}
    \label{fig:UQ}
\end{figure}
which, if smaller than some user-defined threshold $\varepsilon$, would imply that the state is contained in an $\varepsilon$ ball around the mean state estimate. However, as the state is not observable, we rely on probabilistic statements using conformal prediction. By Assumption 1, we can propagate all particles through the nonlinear function \eqref{eq:localizability} to obtain i.i.d. samples $\rho(\xb^{(1)}), \dots, \rho(\xb^{(N)})$ sorted in non-decreasing order. Then, by application of Lemma \ref{lemma:CP}, we obtain $\hat{\varepsilon} = \rho(\xb^{(r)})$ with $r = \lceil (N+1)(1 - \delta_l)\rceil$ such that 
\begin{align}
    \mathrm{Pr}\left[\lVert \xb - \mathbb{E}\{\xb\}\rVert_2 \leq \hat{\varepsilon}\right] \geq 1 - \delta_l
\end{align}
given a desired probability threshold $\delta_l$ on the localization of the state. Finally, we define our measure of uncertainty of the particle filter belief $\bb$ as
\begin{align}
    \mathcal{R}_{\varepsilon}(\bb) = \hat{\varepsilon}(\bb) - \varepsilon
\end{align}
with a desired size of the ball $\varepsilon$. In turn our measure satisfies the property $\mathcal{R}_{\varepsilon}(\bb) \leq 0 \implies \mathrm{Pr}\left[\lVert \xb - \mathbb{E}\{\xb\}\rVert_2 \leq \varepsilon\right] \geq 1 - \delta_l$. We illustrate our proposed uncertainty quantification on a toy example.
\begin{example}
\label{exm:toy}
    Consider a robot operating in one dimension with dynamics $\text{d}x = u~\text{d}t + 0.3~\text{d}W$ where the robot is moving to the right with $u = 0.1$ as shown in Fig. \ref{fig:UQ}. At $|x| \leq 0.1$ is a sensing region that provides a binary measurement, i.e. $z = 1$ if $|x| \leq 0.1$ and $z=0$ otherwise. The initial state distribution is given as a uniform distribution $\mathcal{U}(-1, 1)$. We use $N=2000$ particles to approximate the true posterior and, thus, obtain the constant entropy approximation $\mathcal{H}\approx 7.6$ which does not capture any meaningful information. In Fig. \ref{fig:UQ}, we further visualize our proposed uncertainty quantification metric $\mathcal{R}_{0.1}$ which shortly after 4 seconds (after the robot receives a measurement) becomes negative meaning that the true state is contained within a ball $\mathcal{B}_{0.1}(\mathbb{E}\{x\})$ around the mean state with a desired probability of 95\% . We also plot the empirical standard deviation over time which follows a similar trend as $\mathcal{R}_{0.1}$ but does not admit a statement about the true state in the non-gaussian case.
\end{example}

\section{Belief Space Lyapunov Control}
In this section, we propose our Lyapunov inspired controller for information gathering. Thus, we will tackle the problem of reaching the goal set in state space with desired probability by designing a Lyapunov function in belief space that enforces convergence to a \textit{set of localized beliefs}. 

Let us first formally define the goal of information gathering: we seek to find control inputs $\ub$ such that the belief $\bb$ will end up in the belief goal set
\begin{align}
    \mathcal{S}_b = \left\{\bb \in \mathcal{B} \mid \mathcal{R}_{\varepsilon}(\bb) \leq 0 \bigwedge \mathbb{B}_{\epsilon}\left(\mathbb{E}\{\xb\} \right)\subseteq \mathcal{S}_g\right\}
\end{align}
which is the set of all belief states for which the true state $\xb$ will be contained in an $\varepsilon$ ball around the mean state estimate with desired probability \textit{and} this ball being contained in the goal set in state space. Thus, we get the relationship $\bb \in \mathcal{S}_b \implies \mathrm{Pr}[\xb \in \mathcal{S}_g] \geq 1 - \delta_l$ which is the reach objective formulated in Problem \ref{prob1}.

When further inspecting the goal set in belief space, we can see that it consists of two conceptually distinct requirements: first, the belief must be sufficiently concentrated so that the true state lies inside an $\varepsilon$–ball around the mean with the desired probability, and second, this ball must itself lie inside the state–space goal set. This observation allows us to decouple the information-gathering objective from the goal-reaching objective. The first requirement concerns only the uncertainty of the belief, whereas the second concerns the location of the mean state estimate.

This decoupling enables a modular control design. We can construct a controller whose sole purpose is to drive the belief toward informative regions, that is, regions where $\mathcal{R}_{\varepsilon}(\bb)$ becomes negative. This motivates the design of a stochastic Lyapunov function in belief space that ensures convergence toward such informative areas. Independently of this, we can employ a separate state-based controller that drives the mean state estimate toward the goal region $\mathcal{S}_g$ whenever the uncertainty is sufficiently low.

In the subsequent sections, we formalize this decomposition. We first design an information-gathering controller that renders the uncertainty component of $\mathcal{S}_b$ attractive. Then, we combine it with a goal-reaching controller that enforces progress toward $\mathcal{S}_g$ while gathering information. The resulting modular architecture allows us to guarantee the reach objective with desired probability while keeping information gathering and goal reaching cleanly separated in the control design.

Since information can only be gathered through measurements, we design a controller for a time-discretized system of the hybrid system in Eq. \eqref{eq:hybridsys}. Here, we apply an Euler integration scheme to obtain the stochastic difference equations
\begin{align}
    \bar{\bb}_{k+1} &= \bb_k + \delta t (\bm{f}_b(\bb_k) + \bm{g}_b(\bb_k) \ub_k) + \sqrt{\delta t~\bm{W}_k}\bm{\sigma}(\bb_k) \\
    \bb_{k+1} &= \Delta(\bar{\bb}_{k+1}, \bm{z}_{k+1}):=\bm{F}(\bb_k, \ub_k, \bm{w}_k)\label{eq:discretebeliefdyn}
\end{align}
which we summarize as the overall discrete-time dynamical system $\bb_{k+1} = \bm{F}_b(\bb_k, \ub_k, \bm{w}_k)$ where the random variable $\bm{w}_k$ combines the stochasticity arising from both, the motion noise $\bm{W}_k$ and the stochasticity in resampling after a measurement $\zb_{k+1}$. For this discrete-time dynamical system, we seek to design a stochastic CLF acting on the belief state.

\begin{definition}
\label{def:BCLF}
    A function $W_b :\mathcal{B}\mapsto\mathbb{R}$ is a belief Control Lyapunov Function (CLF) for the the system $\bb_{k+1} = \bm{F}_b(\bb_k, \ub_k, \bm{w}_k)$ if it satisfies
    \begin{align}
        &W_b(\bb) \leq 0 &&\forall \bb: \mathcal{R}_{\varepsilon}(\bb) \leq 0\\
        &W_b(\bb) > 0 &&\forall \bb: \mathcal{R}_{\varepsilon}(\bb) > 0\\
        &\mathbb{E}\left\{W(\bm{F}_b(\bb_{k}, \ub_k, \bm{w}_k))\right\} \leq c W(\bb_k)&&\forall \bb: \mathcal{R}_{\varepsilon}(\bb) > 0, 
    \end{align}
    for some $c \in (0, 1)$.
\end{definition}
% If a function satisfies Def. X, we get the following result.
A Belief CLF provides a systematic way to drive the belief toward regions where the state can be
estimated with the desired accuracy, i.e. a region in belief space in which the true state will be contained in an $\varepsilon$-ball around the mean state. Once such a function is available, we can determine at each
measurement step how to control the discrete-time system to asymptotically reach the set of localized beliefs.
\begin{figure}
    \centering
    \includegraphics[width=0.5\textwidth]{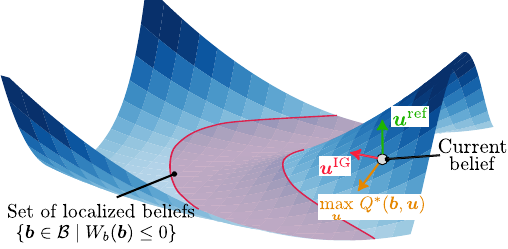}
    \vspace{-0.4cm}
    \caption{Example visualization of a belief CLF. The blue surface depicts the Lyapunov function with the current belief shown in grey. The maximum information gathering control input is the steepest direction on the BCLF, the reference control input points in a direction where the BCLF value increases and the optimized information gathering control input moves in the direction of the reference while enforcing convergence to the set of localized beliefs.}
    \vspace{-0.4cm}
    \label{fig:BCLFVis}
\end{figure}

To implement the goal-reaching component, we assume access to a state–based reference controller
$\ub^{\mathrm{ref}}(\hat{\xb})$ that would reach $\mathcal{S}_g$ if the estimation error were within the
$\varepsilon$–ball. This controller encodes the nominal task objective and is independent of the
belief. The Belief CLF then ensures that we guide the information gathering towards regions that will end up close to the goal. 
We combine these two elements through the optimization problem
\begin{equation}
        \begin{aligned}
        \bm{u}^{\mathrm{IG}} = \arg & \min_{\bm{u}_k \in \mathcal{U}}~   \lVert\bm{u}_k - \bm{u}^\mathrm{ref}\rVert  \\
        \text{s.t.} \quad &\mathbb{E}\left\{W_b(\bm{F}(\bb_{k}, \ub_k, \bm{w}_k))\right\} \leq c W_b(\bm{b}_k) \text{~if~} \mathcal{R}_{\varepsilon}(\bb) > 0.
        \end{aligned}
    \label{eq:LyapControl}
\end{equation}
In words, if the belief state is not inside the goal set in belief space, we try to minimally deviate from the state-based reference control input while still gathering information. If the goal set $\mathcal{S}_b$ is reached, we switch to the reference controller $\ub^{\mathrm{ref}}$. This control strategy is conceptually visualized in Fig.~\ref{fig:BCLFVis}.
\begin{theorem}
Suppose a Belief CLF $W_b$ exists for the belief dynamics in Eq.~\eqref{eq:discretebeliefdyn}, and suppose the
reference controller $\ub^{\mathrm{ref}}$ renders the state–space goal set $\mathcal{S}_g$ reachable
whenever the estimation error is bounded by $\varepsilon$. Then the closed-loop control law
$\ub^{\mathrm{IG}}$ defined in \eqref{eq:LyapControl} ensures
\[
\lim_{k \to \infty} \Pr[\xb_k \in \mathcal{S}_g] \;\ge\; 1 - \delta_l,
\]
i.e. the state converges asymptotically to the goal set in state space with desired probability.
\end{theorem}
\begin{proof}
Since $W_b$ is a Belief CLF, the control law $\ub^{\mathrm{IG}}(\bb_k)$
obtained from~\eqref{eq:LyapControl} ensures that the BCLF decrease condition of
Definition~\ref{def:BCLF} holds whenever $\mathcal{R}_\varepsilon(\bb_k) > 0$.
By Theorem~\ref{thm:SCLF}, this implies $W_b(\bb_k) \to 0$ almost surely as $k \to \infty$ and,
consequently,
\[
\lim_{k \to \infty} \Pr\!\left[\|\xb_k - \mathbb{E}\{\xb_k\}\| \le \varepsilon\right] \;\ge\; 1 - \delta_l
\]
almost surely.
That is, asymptotically, the true state remains in an $\varepsilon$–ball around the mean with
probability at least $1 - \delta_l$.

Whenever $\mathcal{R}_\varepsilon(\bb_k) \le 0$, the assumption on the reference controller
guarantees that $\ub^{\mathrm{ref}}$ drives the mean state estimate $\mathbb{E}\{\xb_k\}$
toward the goal region $\mathcal{S}_g$. Hence, asymptotically the mean
and the corresponding $\varepsilon$–ball around $\mathbb{E}\{\xb_k\}$ are fully contained within $\mathcal{S}_g$. Therefore,
\[
\lim_{k \to \infty} \Pr[\xb_k \in \mathcal{S}_g] \;\ge\; 1 - \delta_l.
\]
\end{proof}
We want to highlight that this decoupling of information gathering and goal reaching allows us to reuse a BCLF for different reach objectives as we only need to ensure a Lyapunov decrease while we minimize the deviation from a reference control input.

\subsection{Finite-Time Convergence}
In many applications, it is desirable to ensure that the state will reach the goal set within finite time. To capture this behavior, we introduce a
finite-time analogue of Belief CLFs that directly parallels the finite-time stochastic CLF
condition of Definition~\ref{def:FSCLF}.

\begin{definition}
\label{def:FTBCLF}
A function $W_b : \mathcal{B} \mapsto \mathbb{R}$ is a finite-time Belief Control Lyapunov
Function (FBCLF) for the system~\eqref{eq:discretebeliefdyn}
if it satisfies
\begin{align*}
    &W_b(\bb) \le 0 
        &&\forall \bb : \mathcal{R}_{\varepsilon}(\bb) \le 0, \\
    &W_b(\bb) > 0 
        &&\forall \bb : \mathcal{R}_{\varepsilon}(\bb) > 0, \\
    &\exists\, \ub \in \mathcal{U} :~
        \Delta W_b \le -\min\big(W_b(\bb),\, \eta\big)
        &&\forall \bb : \mathcal{R}_{\varepsilon}(\bb) > 0,
\end{align*}
for some $\eta > 0$.
\end{definition}

This decrease condition ensures that the value of $W_b$ cannot remain positive indefinitely.
Instead, it must reach zero in finite time, implying that the belief satisfies
$\mathcal{R}_{\varepsilon}(\bb_k) \le 0$ within a finite number of measurement steps. The result
below follows directly from the finite-time stochastic Lyapunov theorem \ref{thm:FSCLF}.

\begin{corollary}
\label{cor:FTreach}
Suppose $W_b$ is a finite-time Belief CLF according to Definition~\ref{def:FTBCLF}, and the
reference controller $\ub^{\mathrm{ref}}(\hat{\xb})$ renders the state-space goal set
$\mathcal{S}_g$ finite-time reachable whenever $\mathcal{R}_{\varepsilon}(\bb) \le 0$. Then there exists a
finite time $T(\bb_0) > 0$ such that
\[
\Pr[\xb_T \in \mathcal{S}_g] \ge 1 - \delta_l
\]
almost surely.
\end{corollary}

This result directly follows from the previous theorem with a reference controller that can drive the state with localization error $\varepsilon$ to the goal set in finite time.

\subsection{Learning a Belief CLF}
\label{sec:learning}
While the previous section has introduced the necessary requirements and their implications, finding a valid belief CLF is notoriously difficult. This is for several reasons: 1) the belief space is of very large dimension, 2) the belief dynamics are highly nonlinear and stochastic and 3) the belief space is unintuitive which complicates the design of handcrafted belief CLFs. Recent literature has shown that certificates such as Lyapunov or Barrier functions can be learned from data. The majority of proposed methods for learning certificates are based on sampling data points from the state space and minimizing a loss that violates the Lyapunov conditions at all data points. While this has proven effective in relatively small state spaces, it is not straightforward to apply these to high dimensional belief spaces. 

Reinforcement learning (RL) offers an appealing alternative. Unlike supervised certificate learning methods that require evaluating Lyapunov conditions across many points in the state space, RL methods learn value functions through trajectories that visit only a small subset of the space. This makes them well suited for high-dimensional systems such as belief dynamics, where exhaustive sampling is impossible.

This motivates a connection between value functions arising in RL and stochastic CLFs. Both quantify progress toward a desirable outcome: a Lyapunov function measures progress toward a goal set by decreasing along system trajectories, while a value function measures expected cumulative reward under a learned policy. The structural similarity between these two notions suggests that an optimal RL value function may itself satisfy the conditions of a stochastic CLF.

To make this connection explicit, consider the standard infinite-horizon discounted RL problem
\begin{equation}
        \begin{aligned}
        V^*(\xb) = ~&\max_{\pi \in \Pi}~ \mathbb{E}\left\{\sum_{k=0}^{\infty} \gamma^k r(\xb_k, \ub_k)\right\}\\
        &\text{s.t.} \quad \xb_{k+1} = \bm{F}(\xb_k, \ub_k, \bm{w}_k),\hspace{0.3cm}\bm{w}_k\sim p(\bm{w})
        \end{aligned}
    \label{eq:RL}
\end{equation}
where we seek to find the optimal policy $\pi^*:\mathcal{X}\mapsto \mathcal{U}$ which maximizes the sum of discounted rewards $r(\cdot)$ given a discount factor $\gamma$.

It is well known that the optimal policy satisfies the stochastic Bellman equation
\begin{align}
    V^*\ofx &= \underset{\ub}{\mathrm{max}}~ \mathbb{E}_w\left\{r(\xb, \ub) + \gamma V^*(\bm{F}(\xb, \ub, \bm{w})\right\} 
\end{align}
where $V^*(\xb)$ denotes the optimal value function. In the following we show the relation between $V^*$ and a stochastic CLF $W$. To the best of our knowledge, this specific result has not been reported in the literature.
\begin{remark}
    Note that we use a general state $\xb$ in our theoretical analysis which aligns with the RL literature. However, the state $\xb$ can also be chosen as the belief state $\bb$.
\end{remark}
\begin{theorem}
\label{thm:BCLF_value}
Let $\gamma \in (0,1)$ and suppose the goal set $\mathcal{S}_g \subset \mathcal{X}$ is absorbing and yields zero reward. 
Assume that outside $\mathcal{G}$, the one-step reward satisfies $r(\xb,\ub) \le R_{\max} < 0$. 
Let $V^*\ofx$ denote the optimal discounted value function of~\eqref{eq:RL} and define $W(x) := -V^*(x)$. 
Then $W(x)$ is a local stochastic Control Lyapunov Function on the set
\begin{align*}
\mathcal{D} = \{\, \xb \notin \mathcal{S}_g \mid 0 < W\ofx \le W_{\max} \,\}
\end{align*}
for any constant 
\begin{align*}
c \in [c_{\min}, 1), \qquad c_{\min} = \frac{1}{\gamma}\!\left(1+\frac{R_{\max}}{W_{\max}}\right) < 1,
\end{align*}
provided $W_{\max} < \tfrac{|R_{\max}|}{1 - \gamma}$. 
\end{theorem}

\begin{proof}
    See Sec. \ref{sec:proof_BCLF}
\end{proof}

\begin{figure}
    \centering
    \includegraphics[]{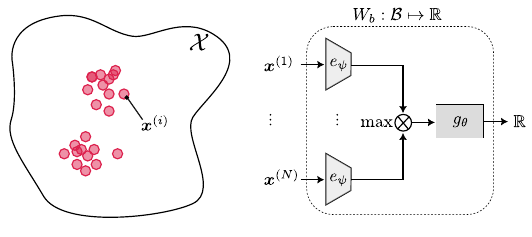}
    \vspace{-0.4cm}
    \caption{Visualization of the network architecture used to learn a belief control lyapunov function. The encoder and MLP are denoted by $e_{\psi}$ and $g_{\theta}$ with learnable parameters $\psi$ and $\theta$, respectively.}
    \vspace{-0.4cm}
    \label{fig:CLFArch}
\end{figure}
\begin{theorem}
\label{thm:FBCLF_value}
    Under the same conditions of Theorem~\ref{thm:BCLF_value}, the function $W\ofx = -V^*\ofx$ is a local fixed-time stochastic CLF on the set
    \begin{align}
        \mathcal{D} = \{\xb \not \in \mathcal{S}_g\mid 0 < W\ofx \leq W_{\max}\}
    \end{align}
    for $\eta \in (0, |R_{\max}|]$ provided that $W_{\max} < \tfrac{|R_{\max}| - \gamma \eta}{1 - \gamma}$.
\end{theorem}
\begin{proof}
    See Sec. \ref{sec:proof_FBCLF}
\end{proof}

These results can be readily extended to algorithms where an optimal action-value function $Q^*(\xb, \ub)$ is available. In this setting, it follows directly from Thm.~\ref{thm:BCLF_value} that $W\ofx = - \underset{\ub}{\mathrm{max}}~Q^*(\xb, \ub)$ is a valid stochastic CLF and the decrease condition can be checked via
\begin{align*}
    \mathbb{E}\{W(\bm{F}(\xb, \ub_k, \bm{w})\} = \frac{1}{\gamma} \left(r(\xb, \ub_k)- Q^*(\xb, \ub_k)\right) \overset{!}{\leq} c W(\xb)
\end{align*}
from which we can directly check if a control input $\ub_k$ satisfies the expected decrease condition or not, without propagating the dynamics.

The theoretical results presented here are central to our method because they provide a principled mechanism for constructing stochastic CLFs for goal-reaching problems. We emphasize, however, that the guarantees hold only for the true optimal value function, which is generally difficult to obtain exactly. In practice, we must rely on approximate value functions, and many off-the-shelf RL algorithms can be used to compute such approximations.

What remains is to show how we can formulate this RL problem on the belief-space model introduced in Eq.~\eqref{eq:discretebeliefdyn}. A direct application of RL is not straightforward because the particle-filter belief representation $p(\xb) \approx \tfrac{1}{N}\sum_{i=1}^N \delta(\xb - \xb^{(i)})$ is invariant to permutations of the particles. Any reordering of the stacked particle vector $\bb = [\xb^{(1)},\dots,\xb^{(N)}]$ represents the same belief distribution, but would be interpreted as a different state by a standard neural network. This permutation symmetry leads to artificial ambiguity in the input representation.

To address this, we adopt a permutation-invariant encoder inspired by architectures for point-cloud processing \cite{qi2017pointnet}. Each particle is first mapped through a shared encoder $e_{\psi}:\mathcal{X}\to\mathcal{O}$ to obtain latent vectors $o^{(i)} = e_{\psi}(\xb^{(i)})$, and we aggregate them with the elementwise maximum
\begin{align}
    \bar{o} = \max_{i=1,\dots,N} \; o^{(i)},
\end{align}
which yields a belief embedding that is invariant to particle ordering. This compact latent representation is then passed through a network $g_{\theta}:\bar{\mathcal{O}}\to\mathbb{R}$ that parameterizes the stochastic CLF. The overall network architecture is illustrated in Fig. \ref{fig:CLFArch}.

To learn this function, we formulate an RL problem over the belief dynamics and choose the reward
\begin{align}
\label{eq:reward}
r(\bb,\ub) = -1 - \mathcal{R}_{\varepsilon}(\bb),
\end{align}
which encourages the agent to minimize the uncertainty measure $\mathcal{R}_{\varepsilon}(\bb)$ and therefore drives the belief quickly toward an informative region. Training the network parameters $(\psi,\theta)$ with off-the-shelf RL algorithms yields an approximate value function whose negative acts as a learned belief-space CLF, as motivated by Theorems~\ref{thm:BCLF_value} and~\ref{thm:FBCLF_value}. Note that if the reach objective changes, we do not need to retrain the BCLF. We only need to retrain if the environment changes.

\section{Belief Space Safety Filter}
In this section, we address the problem of guaranteeing probabilistic safety on the state by designing a safety filter in belief space that corrects the output from the belief CLF if necessary. Since safety is crucial and needs to be ensured at a high frequency, we design our belief space safety filter in continuous time for the stochastic hybrid system in Eq. \eqref{eq:hybridsys}. 

As stated in Problem \ref{prob1}, we seek to provide a probabilistic guarantee that holds over a mission horizon of time $T$, i.e. $\mathrm{Pr}\left[\xb_t \notin \mathcal{S}_a, \forall t\in[0, T]\right] \geq 1 - \delta_a$ which is not to be confused with a probabilistic guarantee pointwise in time $\mathrm{Pr}\left[\xb_t \notin \mathcal{S}_a\right] \geq 1 - \delta_a, \forall t\in [0, T]$. For the latter one, the overall probability of a trajectory being unsafe decreases over time and converges to zero as $t$ increases. To address this problem, we extend recently proposed Belief Control Barrier Functions (BCBFs) which provide risk-aware guarantees pointwise in time, to our desired setting of horizon guarantees.

We first assume that the avoid set can be described as the subzero levelset of a function $h_x$, i.e.
\begin{align}
    \mathcal{S}_a = \left\{\xb \in \mathcal{X} \mid h_x \ofx < 0\right\}
\end{align}
where we assume that $h_x$ is twice continuously differentiable. In the case of a fully observable state, stochastic RCBFs have been proposed to ensure set invariance almost surely as summarized in Sec. \ref{sec:SCBFs}.
\begin{assumption}
    We assume the function $B = \nicefrac{1}{h_x}: \mathcal{X} \mapsto \mathbb{R}$ is a stochastic RCBF according to Def. \ref{def:RCBF} for the fully observable system where the state $\xb$ is known.
\end{assumption}
First, we rewrite the probability of the state being safe within a time interval $I = [t_k, t_{k+1})$ as a single probabilistic statement on the entire past trajectory of the state, i.e.
\begin{align}
    \mathrm{Pr}\left[\xb_t \notin \mathcal{S}_a,~\forall t\in I \right] &= 1 - \delta_a\label{eq:jointCC}\\
    \Leftrightarrow \mathrm{Pr}\left[\underset{\tau \in I}{\mathrm{min}}~ h_x\ofxtau \geq 0\right] &= 1 - \delta_a.\label{eq:min_formulation}
\end{align}
This is true because, if the state is safe within an interval $I$, then its minimum over time will also be safe within the interval $I$. The advantage of using the $\mathrm{min}$ formulation is that we can evaluate a horizon probabilistic constraint by checking a single equation in Eq.~\eqref{eq:min_formulation}.
% introduce a history state $\xi_t = \mathrm{min}(h_x\ofxt, \xi_t)$
\begin{algorithm}[t]
\caption{Horizon BCBF Safety Filter}
\label{alg:conformal_qp}
\begin{algorithmic}[1]
\State \textbf{Input:}  Current particles and their history state $\{(\xb^{(i)}_{t}, \xi^{(i)})\}_{i=1}^{N}$, BCLF control input $\bm{u}^{\mathrm{IG}}$, risk level $\bar{\delta}_a$
\State \textbf{Output:}  Safe control input $\bm{u}$
\State Compute non-conformity scores: 
\[
\rho^{(i)} = -\xi^{(i)}, \quad \forall i=1,\dots,N
\]
\State Sort scores: $\rho^{(1)} \leq \rho^{(2)} \leq \dots \leq \rho^{(N)} \leq \rho^{(N+1)}$
\State Compute threshold index: $p = \lceil (N+1)(1-\delta) \rceil$
\State Select top-$p$ particles: $\xb^{(k)}, \forall k=1,\dots,p$
\State Solve QP in Eq.~\eqref{eq:ConformalQP} with top-$p$ constraints\\
\Return $\ub^*$ from QP solution
\end{algorithmic}
\end{algorithm}

Since, by Assumption~1, we obtain \(N\) i.i.d.\ samples from the belief at time \(t_k\) after the resampling step, propagating each sample forward using an independent realization of the process noise $\text{d}\bm{W}$ yields \(N\) i.i.d.\ continuous-time particle trajectories. Each trajectory
\[
\mathcal{T}_x^{(i)} = \{\, \xb_t^{(i)} \mid t \in [t_k, t_{k+1}) \,\}
\]
is therefore an independent draw from the distribution over state trajectories on the interval \([t_k, t_{k+1})\). This distribution is determined jointly by the initial-state distribution \(p(\xb_{t_k})\) and the law of the noise process \(\{ \bm{W}_\tau \}_{\tau \in [t_k, t_{k+1})}\) through the system dynamics. In other words, the randomness in \(\mathcal{T}_x^{(i)}\) arises solely from the sampled initial condition and the independent noise realization used during propagation, and thus the resulting particle trajectories remain i.i.d.

Therefore, we can use the $N$ particles obtained at time $t \rightarrow t_{k+1}$ right before a measurement update to perform conformal prediction on
the non-conformity score
\begin{align}
    \rho_t = - \underset{\xb \in \mathcal{T}_x}{\mathrm{min}} ~h_x\ofx\label{eq:nonconformityBCBF}
\end{align}
Then, using Lemma 1, we can obtain an upper bound $C$ such that  
\begin{align}
    \mathrm{Pr}\left[\rho(\xb_t) \leq C\right] \geq 1 - \bar{\delta}_a
\end{align}
for some desired probability threshold $\bar{\delta}_a$. If this upper bound would satisfy $C \leq 0$, it follows that Eq.~\eqref{eq:min_formulation} holds which, in turn, would imply that the desired horizon guarantee in Eq.~\eqref{eq:jointCC} holds. Thus, our goal is to find control inputs such that at any time $t \in I$, the conformal upper bound satisfies $C\leq 0$.

To synthesize controls that achieve the desired horizon guarantee, we introduce a history state $\xi$ that keeps track of the running minimum for each particle, i.e.
\begin{align}
    \xi_t^{(i)} = \mathrm{min}(\xi_t, h_x (\xb_t^{(i)}))
\end{align}
which is initialized with $\xi_{t_k} = h_x (\xb_{t_k})$. At any time $t\in I$, we use the history state $\xi$ to calculate the nonconformity score in Eq.~\eqref{eq:nonconformityBCBF} as $\rho_t^{(i)} = - \xi_t^{(i)}$ for all particles $i=1,\dots, N$. We sort the scores and append $\rho^{(N+1)} = \infty$ and perform CP according to Lemma 1, to obtain the upper bound $C$ and all top-$p$ particles for which $\rho_t^{(i)} \leq C = \rho_t^{(p)}$ with $p = \lceil(N+1)(1-\bar{\delta}_a)\rceil$.

Next, we synthesize controls by solving a QP that minimizes the deviation from information gathering control input $\ub^{\mathrm{IG}}$ generated by the BCLF while ensuring that each of the top-$p$ particles satisfies a per particle stochastic RCBF, reading
\begin{equation}
    \begin{aligned}
        \min_{\bm{u} \in \mathcal{U}} \quad & \lVert\bm{u} - \ub^{\mathrm{IG}}\rVert_2\\
        \textrm{s.t.} \quad & \frac{\partial B}{\partial \bm{x}}\Bigg\vert_{\xb^{(i)}} \left(\bm{f}(\xb^{(i)})) + \bm{g} (\xb^{(i)}) \bm{u}\right) \\
        +\frac{1}{2} &\mathrm{tr}\left[\bm{\sigma}^T \frac{\partial^2 B}{\partial \bm{x}^2}\Bigg\vert_{\xb^{(i)}} \bm{\sigma}\right]
      \leq \alpha_3\left(h_x(\xb^{(i)})\right),  \forall i=1,\dots,p.
    \end{aligned}
    \label{eq:ConformalQP}
    \end{equation}
Note that $p$ is a function of the current belief and the induced non-conformity scores.
We summarize the overall controller in Alg.~\ref{alg:conformal_qp}. Following this control strategy, we get the following guarantee.
\begin{theorem}
\label{thm:BCBFQP}
    Given the belief system with dynamics in Eq.~\eqref{eq:BeliefDynamics}, if the QP in Eq. \eqref{eq:ConformalQP} is always feasible on an interval $I=[t_k, t_{k+1}]$, then $\mathrm{Pr}[\xb \notin \mathcal{S}_a, \forall t \in I] \geq 1 - \bar{\delta}_a$.
\end{theorem}
\begin{proof}
    If the QP is always feasible, then all constraints are satisfied and we know by Thm. \ref{thm:SCBF} that the $p$ best particles are always safe, i.e. $h_x(\xb_t^{(k)}) \geq 0$ for all $k = 1,\dots, p$. Consequently the minimum over their trajectory is also positive and $\rho_t^{(k)} \leq 0$ for all $k=1,\dots,p$. In particular, by Lemma \ref{lemma:CP}, it follows that at time $t=t_{k+1}$ the conformal upper bound satisfies $C = \rho_t^{(p)} \leq 0$ from which it follows that
    \begin{align}
        \mathrm{Pr}\left[\underset{\tau \in I}{\mathrm{min}}~ h_x\ofxtau \geq 0\right] &= 1 - \bar{\delta}_a
    \end{align}
    which concludes that the continuous state evolution over the time interval $I$ is safe with desired probability $1-\bar{\delta}_a$.
    % , so in the end the non-conformity score of the true state is smaller than 0 with desired probability, which converts to a horizon guarantee on safety on the interval $I$. 
\end{proof}
This result shows that we can guarantee probabilistic safety within a continuous time interval between measurements with desired probability. Next we show, how this safety guarantee can be extended to longer time intervals. 
\begin{assumption}
\label{ass:transition}
    The belief $\bb$ does not enter the set of unsafe belief states under a stochastic discrete update, i.e.
    \begin{align}
        C(\bb_{t_k^-}) \leq 0 \implies C(\bb_{t_k^+}) \leq 0
    \end{align}
    where $C$ is the conformal upper bound and $t_k^-$ before a measurement update is infinitesimal smaller than $t_k^+$.
\end{assumption}
This assumption states that if the belief is probabilistically safe before a measurement, it will also remain probabilistically safe after a measurement. One example of a discrete transition not satisfying assumption~\ref{ass:transition} is when unsafe particles would get resampled more than $p$ times such that suddenly most probability mass is inside the unsafe set. However, in simulations and experiments we have not observed this case.
To bound the probability of violating safety constraints over the time horizon $T$ we first observe that we can find a lower bound on this probability through the following proposition. 
\begin{proposition}
    \label{prop1}
    Under the same conditions as in Theorem~\ref{thm:BCBFQP}, for a finite time interval $[0, T] = \cup_{k=1}^M I_k$, we have that
    \begin{align}
        \mathrm{Pr} [\xb_t \notin \mathcal{S}_a \forall t \in [0, T]] \geq (1 - \bar{\delta}_a)^M.
    \end{align}
    where $M$ is the number of measurements taken in the interval [0, T].
\end{proposition}
\begin{proof}
    See Sec.~\ref{sec:proofprop}
\end{proof}
Note that from this result, the total probability of the state entering the unsafe set over an arbitrary interval can be bounded by adjusting the $\bar{\delta}_a$. That is, if we want $\mathrm{Pr}[\xb_t \notin \mathcal{S}_a, \forall t \in \cup_{k=1}^M I_k]\geq 1- \delta_a$, we can set $\bar{\delta}_a = 1 - \delta_a^{\nicefrac{1}{M}}$.

\subsection{Guarantees on the Combined Controllers}
In summary, we have introduced three controllers: a state-based reference controller, a belief CLF that enforces the desired information gathering behavior, and a belief CBF that acts as a safety filter to provide probabilistic safety guarantees over a finite horizon. The BCLF and BCBF each come with their own guarantees; probabilistic goal reaching and probabilistic safety, respectively. However, their conditions in Eq.~\eqref{eq:LyapControl} and Eq.~\eqref{eq:ConformalQP} may conflict. For example, the control input required by the BCLF may steer the system toward the unsafe set $\mathcal{S}_a$, causing the safety filter to intervene and potentially leaving the belief dynamics trapped in a local minimum.

To mitigate this, we can use the learned BCLF as a monitor of the IG progress. If the BCLF value stagnates, i.e. it becomes non-decreasing, this indicates a conflict between information gathering and safety. In such cases, we can resample a new IG control input that still satisfies the Lyapunov decrease condition. While this mechanism does not provide a formal guarantee that all conflicts will be resolved, one of our ablation studies demonstrates that it can successfully escape local minima in practice.

An alternative approach would be to combine the BCLF and BCBF constraints into a single optimization problem, thereby eliminating the need for the additional conflict-resolution mechanism. We choose not to take this approach for two reasons. First, the BCLF and BCBF do not need to operate at the same control frequency. Second, retaining the simple QP structure of the BCBF is advantageous for running the safety filter at high control rates, while the BCLF controller is solved at a lower rate determined by the measurement updates. This separation preserves computational efficiency without sacrificing the modularity of the overall architecture.

\section{Evaluation}
We evaluate our proposed control architecture for solving reach-avoid POMDPs in simulations as well as hardware experiments.
Specifically, we
\begin{enumerate}
    \item show improved performance in terms of goal reaching and safety over existing MCTS baselines which merge all control objectives in a single tree search in Sec.~\ref{sec:quanteval},
    \item demonstrate that the proposed minimum deviation control in the BCLF results in goal-oriented information gathering behavior which results in shorter path lengths compared to a pure RL approach for information gathering in Sec.~\ref{sec:AblationCoeff},
    \item show how the learned BCLF can be used to resolve conflict between information gathering and safety in Sec.~\ref{sec:conflict} and ablate that it can be reused for a completely new task in Sec.~\ref{sec:ablationReuse},
    \item analyze a learned BCLF and its theoretical properties in a toy example consisting of two particles which allows to visualize the belief space in Sec.~\ref{sec:analysis}, and
    \item validate our control architecture on a space robotics hardware platform in which a robot has to accurately localize itself by using IMU impact detections in Sec.~\ref{sec:hardware}. 
\end{enumerate}

\subsection{Simulation Setup}
We evaluate our method on three different continuous reach-avoid POMDPs illustrated in Fig.~\ref{fig:envs}.

\textbf{Constrained Lightdark:} In this adaption of the original Lightdark problem~\cite{sunberg2018online}, a robot starts from an unknown initial state $x_0\sim \mathcal{U}(-10, 10)$ and has to reach a goal area located at $\mathcal{S}_g =\{x\in \mathbb{R}\mid |x| < 1\}$ while avoiding the unsafe area at $\mathcal{S}_a = \{x\in \mathbb{R}\mid x \geq 10\}$. The robot's continuous dynamics are described by $\text{d}x = u ~\text{d}t + 0.01 ~\text{d}W$ and measurements happen at a frequency of 5Hz with an observation model $z = x + \sigma(x)$ with spatially varying noise. The most accurate measurements can be obtained right at the boundary to the avoid set.

\textbf{Constrained Antenna:} In this extension, a robot has to operate in 2D to reach the goal set $\mathcal{S}_g$ and avoid the unsafe set $\mathcal{S}_a$ as illustrated in Fig.~\ref{fig:envs}. The robot follows single integrator dynamics $\text{d}\xb = \ub ~\text{d}t + 0.04 ~\text{d}\bm{W}$ and the measurement model is given by $z = \lVert \xb - \ell\rVert_2 + \sigma(\xb)$ which measures the distance to the antenna with spatially varying noise. Due to the distance measurement, the belief can be highly non-Gaussian.

\textbf{Constrained Bumper} In this environment, a robot is moving in a 2D map and can only localize itself by bumping into the wall. Thus, the measurement model is a simple binary measurement $z=\{0, 1\}$ where $1$ indicates a bumper measurement. Further, the bumper sensor has a failure probability of $1\%$ meaning that it can provide the wrong sensor reading. The stochastic dynamics are the same as in the Constrained Antenna environment. The reach and avoid regions are illustrated in Fig.~\ref{fig:envs}.
\begin{table*}[t]
    \centering
    \renewcommand{\arraystretch}{1.4}
    \setlength{\tabcolsep}{10pt}
    \caption{Quantitative Evaluation across environments. For the MCTS baselines, $\delta t$ refers to the planning time step.}
    \label{tab:comparison_detailed}
    \begin{tabular}{l|cc|cc|c|c|c|c}
        \toprule
        & \multicolumn{2}{c|}{CPOMCPOW} & \multicolumn{2}{c|}{CPFT-DPW} & $\ub^{\mathrm{ref}}$ & $\ub^{\mathrm{ref}}$+BCBF & $\ub^{\mathrm{ref}}$+BCLF & Ours \\
        \hline \hline
        \multicolumn{9}{l}{\textbf{Constrained Lightdark}} \\
        \quad Reach $\uparrow$ & 0.72 ($\delta t{=}1$) & 0.52 ($\delta t{=}0.1$) & 0.9 ($\delta t{=}1$) & 0.66 ($\delta t{=}0.1$) & 0.51 & 0.51 & 1.0 & 0.99 \\
        \quad Avoid $\uparrow$   & 0.69 ($\delta t{=}1$) & 0.86 ($\delta t{=}0.1$) & 0.35 ($\delta t{=}1$) & 0.79 ($\delta t{=}0.1$) & 1.0 & 1.0 & 0.03 & 1.0 \\
        \quad SR $\uparrow$   & 0.4 ($\delta t{=}1$) & 0.47 ($\delta t{=}0.1$) & 0.35 ($\delta t{=}1$) & 0.56 ($\delta t{=}0.1$) & 0.51 & 0.51 & 0.03 & \textbf{0.99} \\
        \hline
        \multicolumn{9}{l}{\textbf{Constrained Antenna}} \\
        \quad Reach $\uparrow$& \multicolumn{2}{c|}{0.0 ($\delta t{=}0.1$)} & \multicolumn{2}{c|}{0.0 ($\delta t{=}0.1$)} & 0.53 & 0.05 & 0.97 & 0.95 \\
        \quad Avoid $\uparrow$   & \multicolumn{2}{c|}{1.0 ($\delta t{=}0.1$)} & \multicolumn{2}{c|}{0.99 ($\delta t{=}0.1$)} & 0.37 & 0.98 & 0.41 & 0.98 \\
        \quad SR  $\uparrow$ & \multicolumn{2}{c|}{0.0 ($\delta t{=}0.1$)} & \multicolumn{2}{c|}{0.0 ($\delta t{=}0.1$)} & 0.33 & 0.05 & 0.4 & \textbf{0.95} \\
        \bottomrule
    \end{tabular}
\end{table*}

For our initial comparison, we use a finite-time BCLF according to Def.~\ref{def:FTBCLF} and a simple proportional reference controller that would guide the state to the goal region if the state would be perfectly known. To learn a BCLF, we train an RL agent in each environment with the proposed reward structure in Eq.~\eqref{eq:reward} using deep Q-networks (DQNs)~\cite{mnih2015human}. To that end, we discretize the control space and learn a Q function. In Constrained Lightdark the action space is reduced to 7 equidistant actions between -10 and 10 and in the other environments, we discretize the action space into 9 actions which are the move directions in cardinal directions at constant velocity. Note that we only discretize the action space to make the RL problem simpler, the safety filter, e.g., can still apply continuous actions to the robot. For the encoder $e_{\psi}$ we use simple MLP with two hidden layers of dimension 32 and a latent dimension of 8. For the Q function, we use an MLP with three hidden layers of dimension 256 (128 in Constrained Lightdark). The goal sets in belief space are chosen such that at convergence the state will be localized in a ball $\mathbb{B}_{\varepsilon}$ with probability $1-\delta_l=0.99$. The size of this ball is chosen such that it can be fully contained in the goal region.
\begin{figure}
    \centering
    \includegraphics[width=0.48\textwidth]{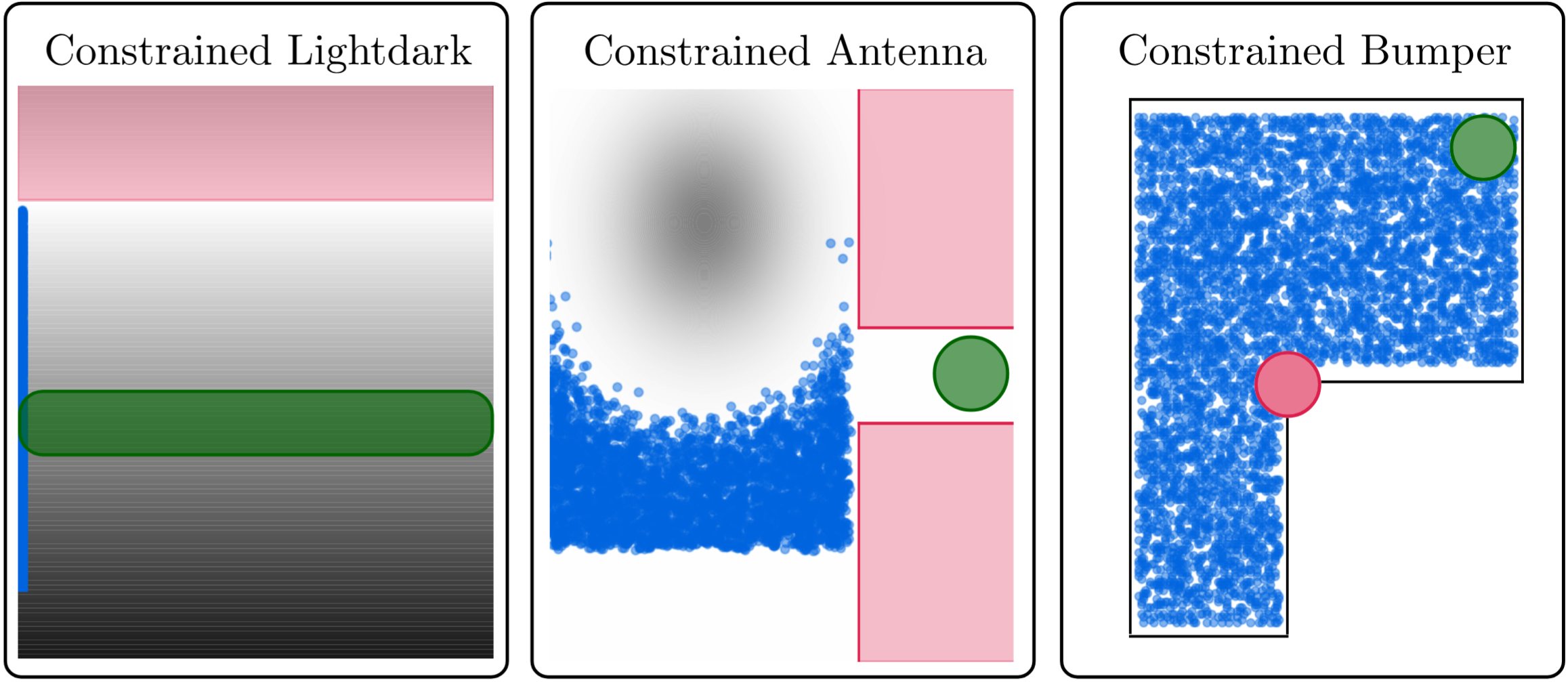}
    \vspace{-0.4cm}
    \caption{Illustration of the three reach-avoid POMDPs considered in simulation. Goal regions $\mathcal{S}_g$ are shown in green, avoid regions $\mathcal{S}_a$ are shown in red and the particle filter belief is shown in blue.}
    \vspace{-0.4cm}
    \label{fig:envs}
\end{figure}
The RL agents are trained with a discount factor of $\gamma = 0.99$ and the BCBF safety filter has a per interval safety probability of $1-\bar{\delta}_a = 0.99$ in Constrained Lightdark and Constrained Antenna. The safety probability in the Bumper environment is set to $1-\bar{\delta}_a=0.9$ because the robot would not move otherwise based on the initial belief. For the BCBF we use a stochastic CBF of the form $h_x(x) = 10 - x$ in Constrained Lightdark environment. In Constrained Antenna, we use a continuously differentiable overapproximation of the unsafe set using smooth $\min$ and $\max$ operators~\cite{molnar2023composing} and in Constrained Bumper, the stochastic CBF is chosen as $h_x = \lVert \xb - \mathcal{O}\rVert_2 - r_{\mathcal{O}}$ which is the distance to the unsafe region.
\begin{figure*}
    \centering
    \includegraphics[width=\textwidth]{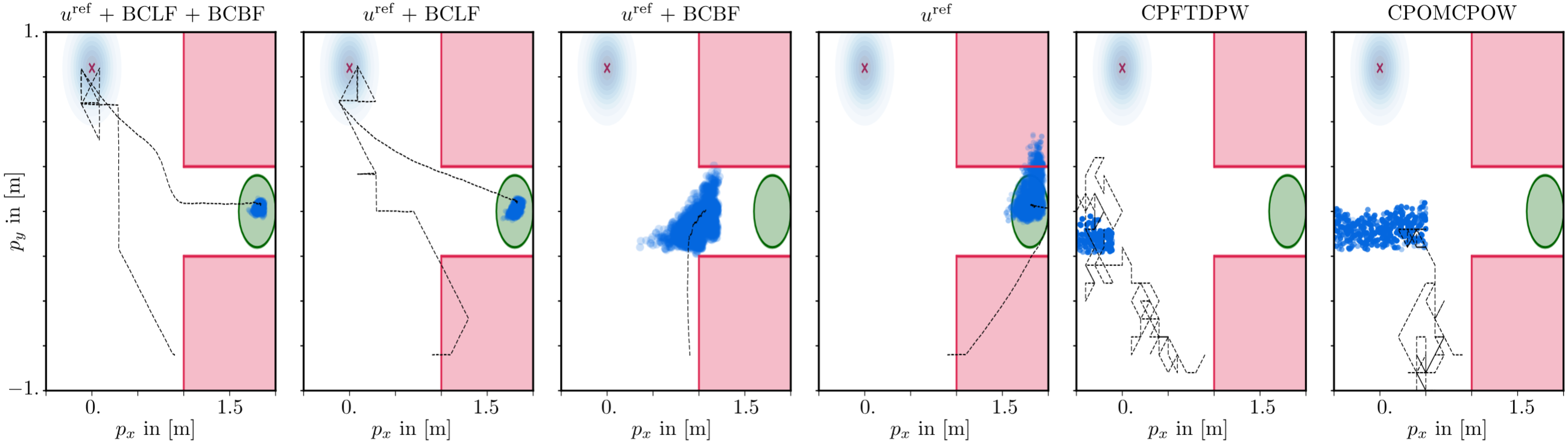}
    \vspace{-0.4cm}
    \caption{Illustration of all baselines in the Constrained Antenna environment in which the antenna location is indicated by a red cross, the goal region is shown in green and the avoid set is shown in red. In this run, only our proposed control architecture is able to reach the goal without entering the set of unsafe states.}
    \vspace{-0.2cm}
    \label{fig:AntennaSim}
\end{figure*}

\subsection{Baselines}
First, we compare against different combinations of our proposed control architecture to isolate the effect of each individual module. Specifically, we compare against the pure reference-based controller that only acts on the mean state ($u^{\text{ref}}$), the reference controller with a BCBF ($u^{\text{ref}}$ + BCBF) and the reference controller with a BCLF but without a safety filter ($u^{\text{ref}}$ + BCLF). Further, we compare against state-of-the-art solvers for constrained POMDPs based on Monte Carlo Tree Search. Namely, we compare against constrained partially observable planning with observation widening (CPOMCPOW) and constrained particle filter trees with double progressive widening (CPFT-DPW)~\cite{jamgochian2023online}. These baselines try to maximize expected rewards subject to expected cost constraints. Since this is different from our original problem formulation, we use the negative distance to the goal set as reward function and the distance to the unsafe set as cost function. We use the same hyper parameters as proposed in~\cite{jamgochian2023online}. We further compare to two different versions of the MCTS baselines in which we change the time step used in rollouts. This allows us to analyze the reach and avoid behavior for different horizon lengths and control frequencies.

\subsection{Quantitative Evaluation}
\label{sec:quanteval}
First, we consider the environments Constrained Lightdark and Constrained Antenna. We run 100 simulations for each environment in which the true initial state is randomized. The quantitative results of the simulation study are summarized in Table~\ref{tab:comparison_detailed}. It can be seen that our method achieves the overall best results when looking at the success rate, i.e. the runs in which the state reaches the goal and stays safe along the entire trajectory. In the Constrained Lightdark environment we can observe that for both MCTS baselines the reach rate increases and the avoid rate decreases if we reduce the time step. This can be attributed to a shorter planning horizon in the case of $\delta t=0.1$ such that information gathering becomes more difficult. On the other hand, if the time step is lower, the belief can be corrected at higher frequencies which results in higher avoid rates. This underlines our claim that different components occuring in POMDP solving do not necessarily all need to run at the same frequency. 

We can observe that all baselines with the proposed BCBF keep the state safe in almost all simulation runs showing that the belief safety filter successfully keeps the majority of particles outside $\mathcal{S}_a$. Figure~\ref{fig:AntennaSim} shows examplary runs for all baselines in the Constrained Antenna environment. It can be observed that for $u^{\text{ref}} + \text{BCBF}$, the belief is kept safe but it is trapped in a local minimum since the belief is not accurate enough to safely reach the goal. It also becomes apperent that the MCTS baselines try to maximize the reward by reducing the distance to the goal in the $p_y$ direction but the belief is not sufficiently accurate to reach the goal without entering the unsafe set resulting in overly conservative behavior. The baseline $u^{\text{ref}} + \text{BCLF}$ efficiently gathers information and reaches the goal but becomes unsafe during the information gathering process. 
Further, we can see that the baselines with a BCLF and a reference controller have a high reach rate since the robot gathers information first before approaching the goal. Combining all modules into one control architecture leads to high reach and avoid rates. Further, we want to highlight that IG and goal reaching are not aligned in these two environments, i.e. in order to gather information we need to move away from the goal. However, since the BCLF wants to reach an accurately localized belief state in finite time, the belief is first localized before moving to the goal.

\subsection{Ablation: Lyapunov Coefficient}
\label{sec:AblationCoeff}
\begin{figure*}
    \centering
    \includegraphics[width=\textwidth]{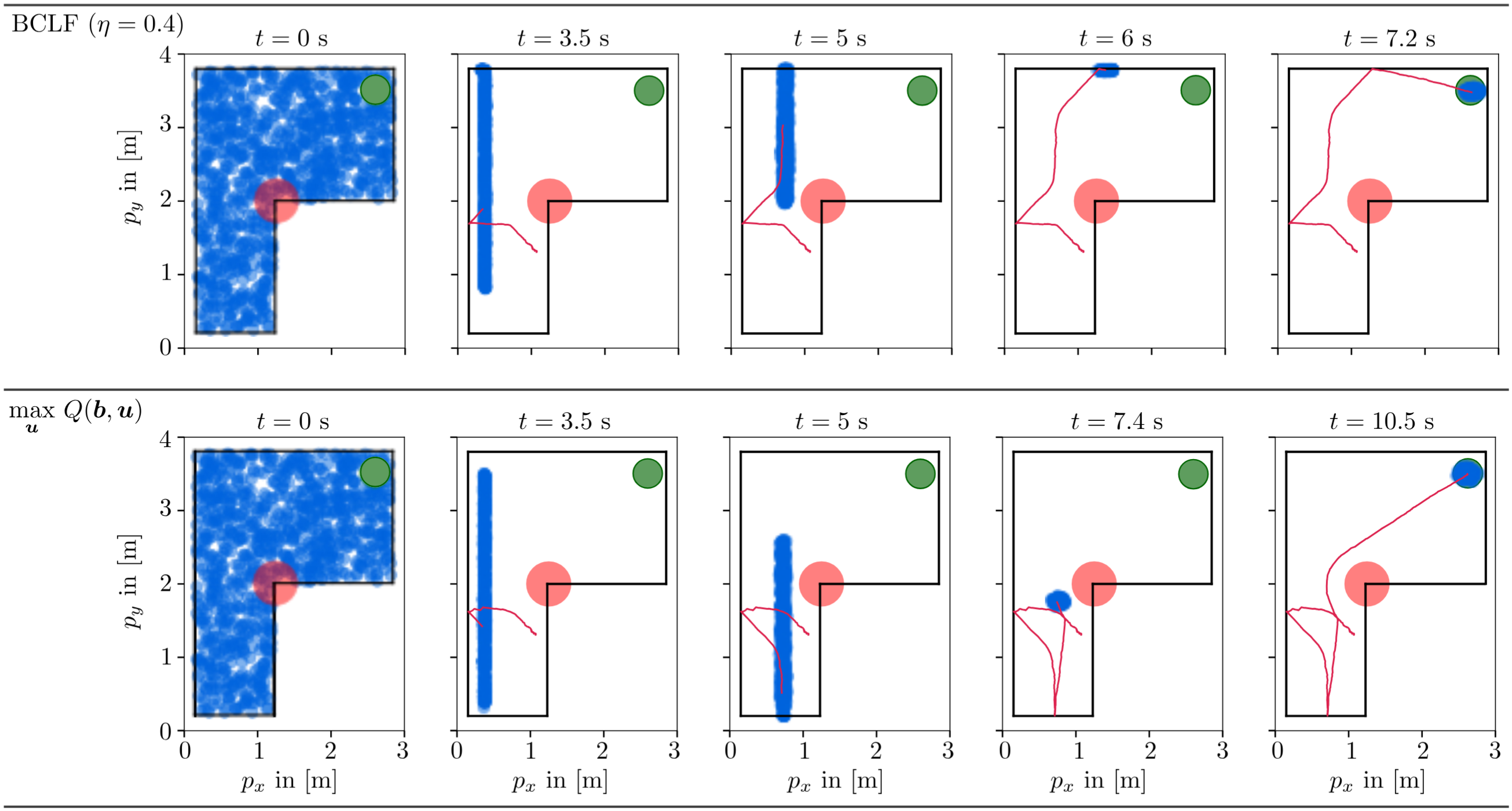}
    \vspace{-0.4cm}
    \caption{Illustration of the proposed control architecture (top) and the switching controller (bottom) in the Constrained Bumper environment. Our method first gathers information by bumping into the left wall and then continues towards the goal as uncertainty can be reduced by bumping into the top wall. The switching controller moves south after the first bump resulting in overall longer path length.}
    \label{fig:ComparisonSC}
    \vspace{-0.5cm}
\end{figure*}
Next, we analyze the influence of the Lyapunov coefficients $c$ in the asymptotic BCLF and $\eta$ in the finite-time BCLF, respectively. To that end, we shift our focus to the Constrained Bumper environment since, in contrast to the other environments, the robot has multiple different choices of gathering information, i.e. bumping into the top/bottom  or left/right walls. We aim to show that the minimum deviation control proposed in Eq.~\eqref{eq:LyapControl} yields goal oriented information gathering. To that end, we also compare to another baseline based on a switching control (SC) policy, i.e. the robot follows the information gathering policy until it is accurately localized and then switches to the state-based reference. The unfiltered control strategy thus results in
\begin{align*}
    \ub_{\text{SC}} = \mathds{1}_{\mathcal{R}_{\varepsilon}(\bb) > 0} ~\underset{u}{\mathrm{max}}~ \mathbb{E}\{V^*(\bm{F}(\bb, \ub, \bm{w})\} + \mathds{1}_{\mathcal{R}_{\varepsilon}(\bb) \leq 0}~ \ub^{\mathrm{ref}}.
\end{align*}
We run 100 simulations in the Constrained Bumper environment and report the reach and avoid rates as well as the path length accross successful runs (PL).
We first note in Fig.~\ref{fig:boxplots_merged} that as we increase $c$ in the asymptotic BCLF and decrease $\eta$ in the finite time BCLF, the overall path length reduces. This is caused by the constraint formulation that we do not require a fast decrease on the BCLF which allows the robot to gather information at a later point in time which is closer to the goal region. We visualize this qualitatively in Fig.~\ref{fig:ComparisonSC} in which we show two runs with equivalent initial conditions for the switching controller and a BCLF. It can be seen that the switching controller first bumps into the left wall and then navigates south to accurately localize. In contrast, our BCLF enforces a behavior in which the robot knows that it can gather information on the way to the goal which results in shorter path length and thus faster goal reaching time. For both runs, the BCBF keeps the robot safe. Thus, the Lyapunov coefficent serves as a tuning knob to interpolate between information gathering and goal reaching behavior.

However, we also note that increasing the information gathering behavior too much results in lower reach and avoid rates as shown in Table~\ref{tab:PLA_singlecol_PL}. This can be attributed to the fact that a too little focus on IG can result in poor information gathering if the learned BCLF is not accurate enough. That is, the robot moves the mean state towards the goal thinking that it can gather enough information in the future which in turn leads to a timeout since the robot does not accurately localize within the simulation time. We can see that this effect is less visible in the finite time BCLF as we impose the robot to accurately localize within finite time. However, also for the setting $\eta=0.1$, the reach rate drops which can be attributed to the same reason. Based on Thm. \ref{thm:FSCLF} the expected time to localize is in less than $K = \lceil \nicefrac{W_b(\bb_0)}{\eta} \rceil \approx 800$ measurements for $\eta=0.1$. In simulation however, we use a timeout of $10$ seconds with an update rate of 5 Hz which means the robot is not expected to localize within simulation time for $\eta=0.1$.
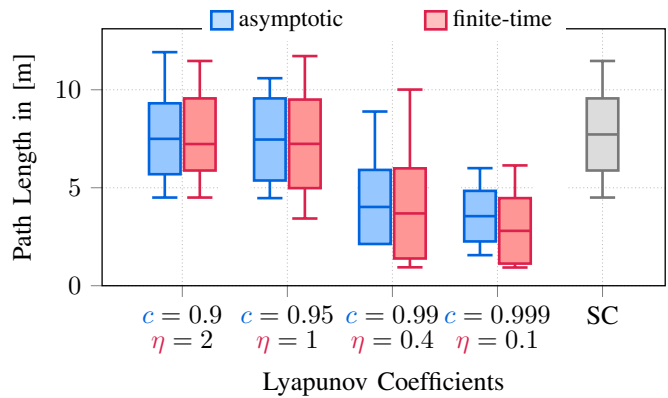
\begin{figure}
\centering
\begin{tikzpicture}
\begin{axis}[
  width=0.5\textwidth, height=5cm,
  ymin=0,
  ylabel={Path Length in [m]},
  xlabel={Lyapunov Coefficients},
  xtick={1.5,4.5,7.5,10.5,13.5},
      xticklabels={
        {\shortstack{$\textcolor{customblue}{c}=0.9$\\$\textcolor{customred}{\eta}=2$}},
        {\shortstack{$\textcolor{customblue}{c}=0.95$\\$\textcolor{customred}{\eta}=1$}},
        {\shortstack{$\textcolor{customblue}{c}=0.99$\\$\textcolor{customred}{\eta}=0.4$}},
        {\shortstack{$\textcolor{customblue}{c}=0.999$\\$\textcolor{customred}{\eta}=0.1$}},
        {SC},
      },
      xticklabel style={align=center},
  grid=both,
  major grid style={densely dotted},
  tick align=outside,
  tick pos=left,
  ymajorgrids,
  boxplot/draw direction=y,
  boxplot/box extend=0.9, % controls box width
  line width=0.5pt,
  boxplot/every box/.style={solid},
  boxplot/every whisker/.style={solid},
  boxplot/every median/.style={solid},
  boxplot/every cap/.style={solid},
  legend style={
  at={(0.5,0.95)},
      anchor=south,
      legend columns=2,
      draw=none,
      /tikz/every even column/.append style={column sep=1cm},
      font=\small
    },
]

% Helper: one prepared box at a specific x-position
% Usage: \bp{<x>}{<draw>}{<fill>}{<median>}{<q1>}{<q3>}{<loww>}{<upw>}
\newcommand{\bp}[8]{%
\addplot+[
  line width=1pt,
  draw=#2,
  fill=#3,
  boxplot prepared={
    median=#4,
    lower quartile=#5,
    upper quartile=#6,
    lower whisker=#7,
    upper whisker=#8,
  },
] coordinates {};
}
\newcommand{\bpforget}[8]{%
\addplot+[
  forget plot,
  line width=1pt,
  draw=#2,
  fill=#3,
  boxplot prepared={
    median=#4,
    lower quartile=#5,
    upper quartile=#6,
    lower whisker=#7,
    upper whisker=#8,
  },
] coordinates {};
}

% -------- Random example data --------
% Setting 1
\bp{1}{customblue}{lightblue}{7.496}{5.69}{9.31}{4.5}{11.92}
\bp{2}{customred}{lightred}{7.23}{5.88}{9.56}{4.5}{11.47}

\bp{3}{white}{white}{5}{5}{5}{5}{5}

% Setting 2
\bp{4}{customblue}{lightblue}{7.46}{5.37}{9.56}{4.47}{10.59}
\bp{5}{customred}{lightred}{7.24}{4.98}{9.5}{3.43}{11.72}

\bp{6}{white}{white}{5}{5}{5}{5}{5}

% Setting 3
\bp{7}{customblue}{lightblue}{4.02}{2.13}{5.91}{2.22}{8.89}
\bp{8}{customred}{lightred}{3.69}{1.39}{5.99}{0.94}{10.01}

\bp{9}{white}{white}{5}{5}{5}{5}{5}

% Setting 4
\bp{10}{customblue}{lightblue}{3.55}{2.26}{4.84}{1.56}{6}
\bp{11}{customred}{lightred}{2.8}{1.13}{4.47}{0.93}{6.14}

\bp{9}{white}{white}{5}{5}{5}{5}{5}
\addplot+[
  line width=1pt,
  draw=customgray,
  fill=lightgray,
  boxplot prepared={
    draw position=13.5,
    median=7.72,
    lower quartile=5.88,
    upper quartile=9.56,
    lower whisker=4.5,
    upper whisker=11.47,
  },
] coordinates {};

% legend image = filled rectangle
\pgfplotsset{
  legend image code/.code={
    \path[draw=#1, fill=#1, fill opacity=0.6]
      (0cm,-0.12cm) rectangle (0.28cm,0.12cm);
  }
}

\addlegendimage{area legend, draw=customblue, fill=lightblue, fill opacity=0.6}
\addlegendentry{asymptotic}

\addlegendimage{area legend, draw=customred, fill=lightred, fill opacity=0.6}
\addlegendentry{finite-time}

\end{axis}
\end{tikzpicture}
\vspace{-0.7cm}
\caption{Box plots showing median, IQR, and min/max values for the path length for both BCLFs and the switching controller (SC).}
\label{fig:boxplots_merged}
\end{figure}
\begin{table}[ht]
\centering
\caption{Results for the switching controller (SC), asymptotic BCLF coefficient $c$, and finite-time BCLF coefficient $\eta$.}
\label{tab:PLA_singlecol_PL}
\renewcommand{\arraystretch}{1.2}
\setlength{\tabcolsep}{5pt}
\begin{tabular}{l l c c c}
\hline
~ & \textbf{Setting} & \textbf{Reach $\uparrow$} & \textbf{Avoid $\uparrow$} & \textbf{PL [m] $\downarrow$} \\
\hline\hline
SC & -- & 0.95 & 0.97 & $7.54 \pm 1.31$ \\
\hline
\multirow{4}{*}{$c$}
& $0.9$    & 0.80 & 0.97 & $7.38 \pm 1.33$ \\
& $0.95$   & 0.77 & 0.97 & $7.32 \pm 1.39$ \\
& $0.99$   & 0.89 & 0.87 & $4.43 \pm 1.48$ \\
& $0.9999$ & 0.74 & 0.91 & $3.65 \pm 0.97$ \\
\hline
\multirow{4}{*}{$\eta$}
& $2$   & 0.96 & 0.97 & $7.19 \pm 1.66$ \\
& $1$   & 0.96 & 0.97 & $7.12 \pm 1.66$ \\
& $0.4$ & 0.97 & 0.88 & $4.22 \pm 1.92$ \\
& $0.1$ & 0.82 & 0.89 & $3.07 \pm 1.19$ \\
\hline
\end{tabular}
\end{table}

\subsection{Using the BCLF to resolve conflict}
\label{sec:conflict}
Next, we illustrate a failure case of the proposed architecture and show how it can be resolved. In certain scenarios, the information-gathering strategy may conflict with the safety filter. One such case is shown in Fig.~\ref{fig:ConflictRes}, where the robot starts from an initial uniform belief and bumps into the right wall. As a result, the belief becomes multimodal, with part of the probability mass already lying inside the goal set. In this situation, the reference controller attempts to move the robot in the northeast direction, which conflicts with the safety filter because doing so would drive a significant portion of the belief through the unsafe set. Consequently, the robot becomes trapped in a local minimum and remains safely stationary until timeout.

Figure~\ref{fig:ConflictRes} further shows that the BCLF value stagnates once the belief is trapped. This observation motivates using the BCLF as a monitor for the information-gathering process and triggering a conflict-resolution mechanism when a local minimum is detected. In this example, if no information is gathered for more than one second, we switch to the direction of maximum information gain, i.e., the direction that maximizes the decrease of the BCLF. This simple mechanism is sufficient to escape the local minimum and successfully reach the goal.
\begin{figure}
    \centering
    \includegraphics[width=0.48\textwidth]{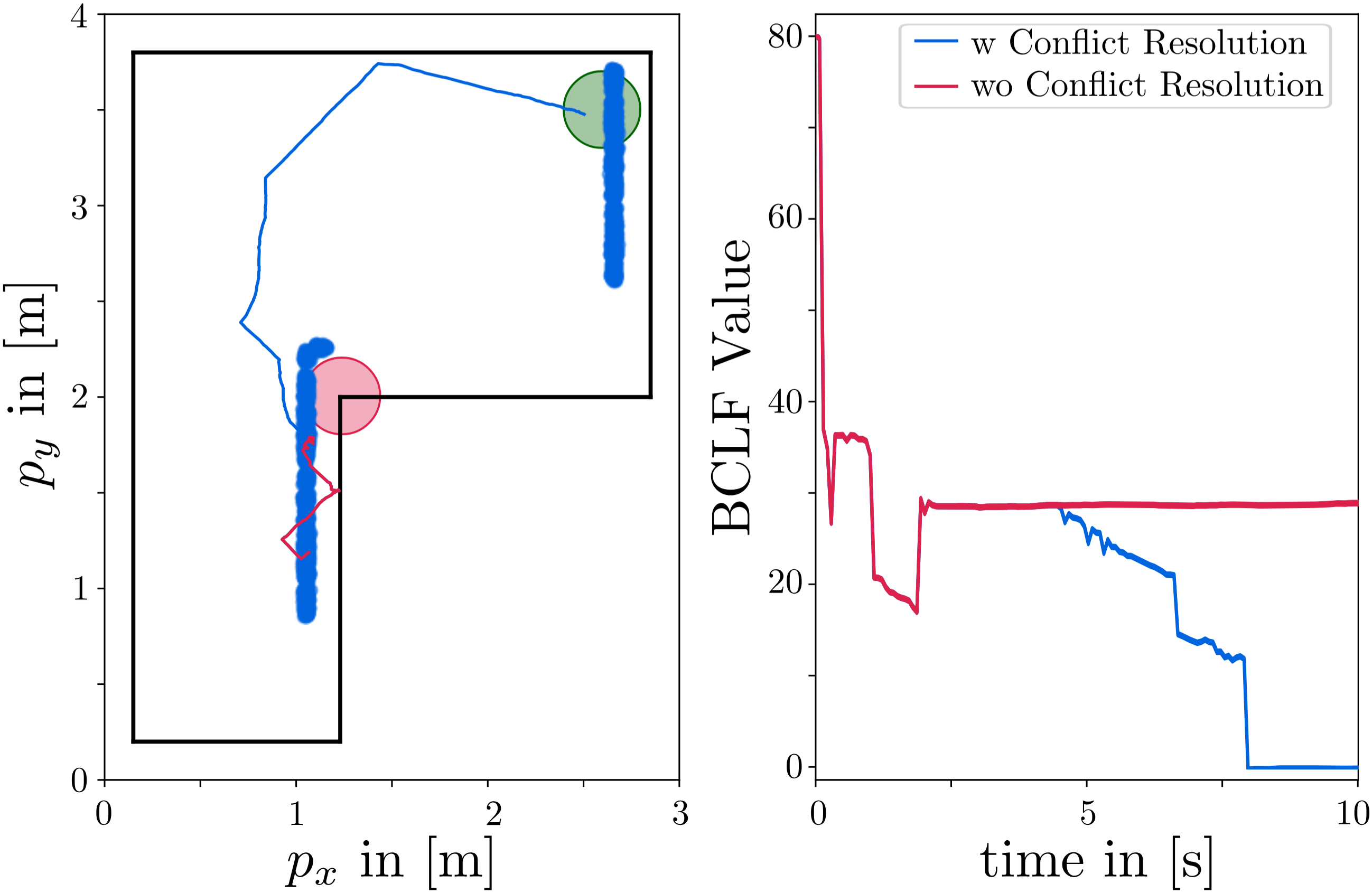}
    \vspace{-0.2cm}
    \caption{Illustration of the conflict resolution in the Constrained Bumper Environment. The particle filter belief being stuck in a local minima is shown by the blue particles. The ground truth trajectory of the state under conflict resolution is shown in blue.}
    \label{fig:ConflictRes}
    \vspace{-0.5cm}
\end{figure}

\subsection{Ablation: Reuse of Learned Lyapunov Function}
\label{sec:ablationReuse}
We further show how the control architecture can be quickly adapted to new tasks. To that end, we formulate a completely different reach-avoid task as illustrated in Fig. \ref{fig:NewTask} where we use the same Constrained Bumper environment but want the robot to perform a circular tracking task while avoiding the unsafe areas. We use exactly the same BCLF and only change the state-based reference controller and the BCBF. We show two different initial beliefs with samples drawn from a Gaussian distribution. Both initial conditions lead to the robot safely fulfilling the desired tracking task. We want to highlight that the depicted Gaussian belief states were never seen during training since an episode always starts from an initial uniform belief as depicted in Fig. \ref{fig:envs}. Thus the learned BCLF is able to generalize to unseen beliefs.
\begin{figure}
    \centering
    \includegraphics[width=0.48\textwidth]{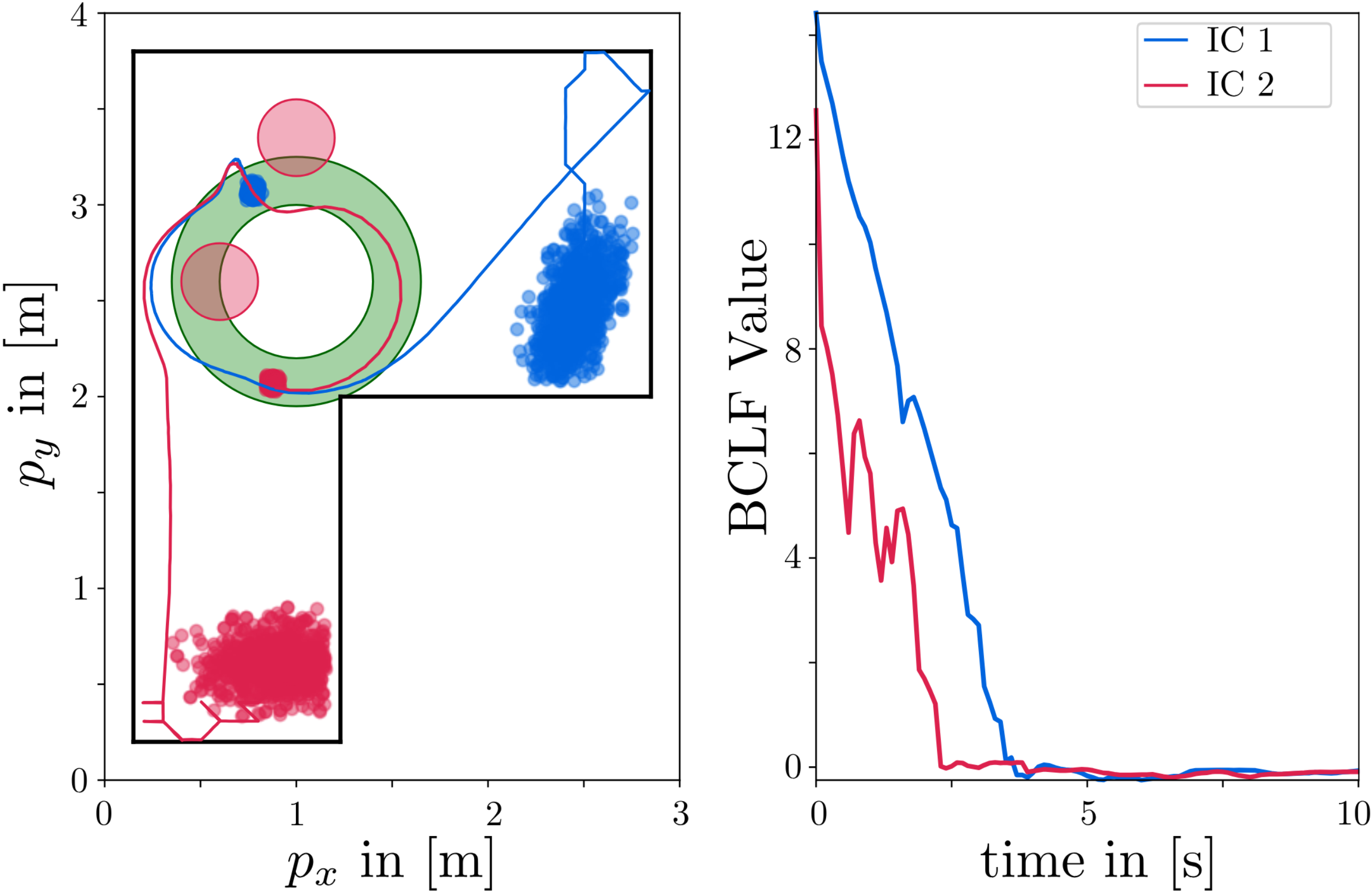}
    \vspace{-0.2cm}
    \caption{Application of the proposed Control Architecture on a different Reach-avoid problem where the robot has to fulfill a circular tracking objective. Different initial conditions (ICs) are shown in blue and red.}
    \label{fig:NewTask}
    \vspace{-0.5cm}
\end{figure}

\subsection{Analysis of Theoretical Properties}
\label{sec:analysis}
In this simulation study, we investigate the learned BCLF on a variation of the example introduced in Example \ref{exm:toy} with a one dimensional state space and a discrete sensing region at $\{x \in \mathbb{R}\mid |x|\leq 1\}$. In order to visualize the learned BCLF, we consider a simplified scenario in which the particle filter belief consists of only two particles $x^{(1)}$ and $x^{(2)}$, respectively. The true state is either one of the particles, so the particle filter belief is reduced to a simple bimodal belief. We consider a time-discretized belief dynamics model using the sampling rate $\delta t = 0.1$ seconds. We train a DQN with a $Q$ function that is parameterized by an MLP with two layers of hidden dimension 256 without the permutation invariant encoder proposed in Sec. \ref{sec:learning}. Since two particles are not enough to apply conformal prediction as in Lemma \ref{lemma:CP}, we define the goal set in belief space as $\mathcal{S}_b = \{\bb \in \mathcal{B}\mid |x^{(1)} - x^{(2)}| \leq 0.1\}$. The discount factor during training is set to $\gamma = 0.99$ and the reward structure is a simple negative reward of $r(\bb, \ub)=-1$ if the belief state is outside of the goal set.

Figure \ref{fig:BCLF} shows the learned BCLF, with the sub-zero level set highlighted in orange. The learned BCLF clearly separates the goal states in belief space, as the sub-zero level set is a subset of the ground-truth set $\mathcal{S}_b$. The optimal information-gathering policy in this scenario is straightforward: if both particles lie to the left or to the right of the sensing region, the optimal action is to move right or left, respectively; if the particles lie on opposite sides, information can be gathered by moving in either direction. This behavior is reflected in the learned BCLF: when the particle signs differ, either direction reduces the BCLF value, whereas when they share the same sign, information can only be gathered by moving toward the sensing region. A $+$-shaped structure emerges in the $\pm 0.1$ band around the particle axes, where at least one particle lies inside the sensing region. In this region the state will be accurately localized after the next measurement, so the value equals the one-step reward of $-1$. The maximum value over the belief space is $W_{\max} = 8.73$, satisfying $W_{\max} \le \nicefrac{|R_{\max}}{1-\gamma} = 100$ for an asymptotic BCLF with minimum Lyapunov coefficient $c_{\min} \approx 0.89$. Thus, the learned BCLF is a valid asymptotic BCLF over the entire bounded belief space. Moreover, $W_{\max} \le \nicefrac{|R_{\max}| - \gamma \eta}{1-\gamma} = 60.4$ for a suitable $\eta = 0.4$, confirming that the finite-time BCLF is also valid on the entire bounded belief space.
\begin{figure}[t]
    \centering
    \includegraphics[width=0.48\textwidth]{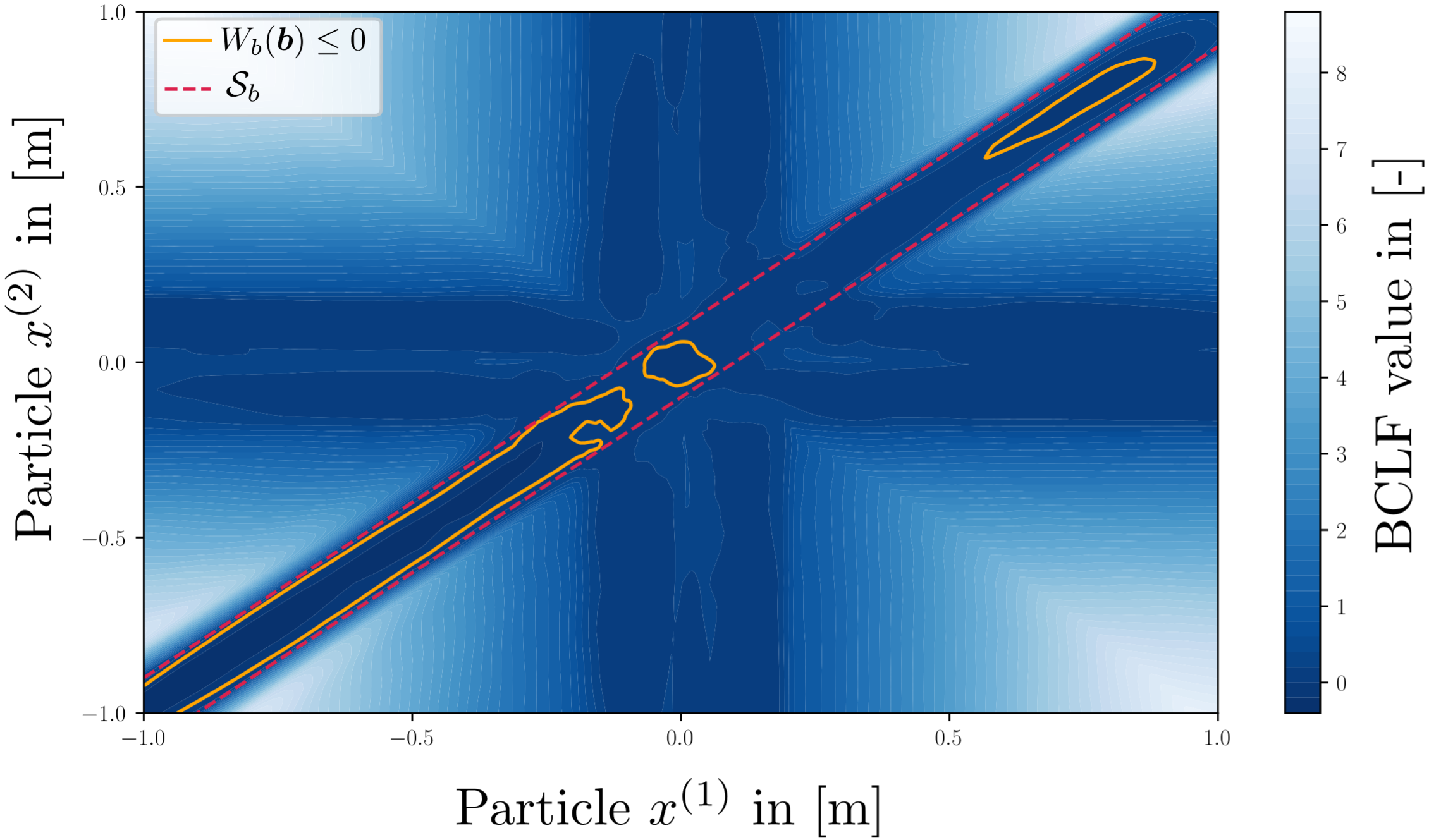}
    \vspace{-0.2cm}
    \caption{Illustration of the learned BCLF for the example two particle system. The colorgradient illustrates higher values of the BCLF, the ground truth goal set is shown in red while the sub-zero levelset of the BCLF is visualized in orange.}
    \vspace{-0.4cm}
    \label{fig:BCLF}
\end{figure}

Beyond visual inspection, we also empirically evaluate the conditions of the BCLFs in Defs.~\ref{def:BCLF} and \ref{def:FTBCLF}. To this end, we sample ten thousand random initial belief states and simulate the system under the optimal information-gathering policy. Table~\ref{tab:BCLFConditions} summarizes the conditions and the corresponding true positive and false positive rates. We find that the subzero level set of the BCLF tends to underapproximate the goal set, as indicated by a true positive rate noticeably below one, which is consistent with the visualization in Fig.~\ref{fig:BCLF}. The BCLF is, however, close to zero for the vast majority of samples inside the goal set $\mathcal{S}_b$. Likewise, positive definiteness holds for almost all samples outside the goal set. Although the BCLF and FBCLF decrease conditions do not hold globally outside the goal set, they remain sufficient to guide the system toward the goal, as all simulated belief trajectories successfully reach $\mathcal{S}_b$. Finally, we observe that the expected upper bound on the hitting time (see Def.~\ref{def:FSCLF}) is overly conservative: most simulations reach the goal set substantially earlier, even though the bound is only required to hold in expectation.

% Sensitivity to hyper parameters.
\subsection{Hardware Experiments}
\label{sec:hardware}
We finally perform hardware validation for a scenario similar to the Constrained Bumper scenario. Here, we consider a space-analogue platform that is able to approximate weightlessness in two dimensions by floating with three air-bearings over an epoxy-coated floor.
During normal operation inside a spacecraft, global localization may be performed using observations of fiducial markers, similar to the Astrobee platform in the International Space Station~\cite{smith2016astrobee}. When these cameras fail, the only available sensor is an onboard IMU. To navigate to its repair station autonomously, the robot needs to interact with its environment similar to the Bumper example.

% On the space platform, we perform an experiment that is similar to the Roomba POMDP, i.e. we consider the scenario where global localization has a failure and the only sensor available is an IMU. In this scenario, we want the robot to navigate to the repair station autonomously. The dynamics model will be given as
\begin{table}[t]
    \centering
    \caption{Performance metrics for function $W_b(\bb)$ across sampled points}
    \begin{tabular}{lcccc}
        \toprule
        Condition & TPR & FPR  \\
        \midrule
        $W_b(\bb) \leq 0$ for $\bb \in \mathcal{S}_b$ & 0.417 & $1e^{-4}$ \\
        $W_b(\bb) > 0$ for $\bb \notin \mathcal{S}_b$ & 0.999 & 0.583 \\
        $W_b(\bb_{k+1}) \leq c W_b(\bb_k)$ for $\bb_k \notin \mathcal{S}_b$ & 0.982 & -- \\
        $\Delta W_b(\bb_k) \leq -\min(W_b(\bb_k), \eta)$ for $\bb_k \notin \mathcal{S}_b$ & 0.95 & -- \\
        Reach $\mathcal{S}_b$ & 1.0 & -- \\
        Reach $\mathcal{S}_b$ within upper bound $T <= \mathbb{E}\{K(\bb_0)\}$ & 0.983 & -- \\
        \bottomrule
    \end{tabular}
    \label{tab:BCLFConditions}
\end{table}
\begin{figure*}
    \centering
    \includegraphics[width=\textwidth]{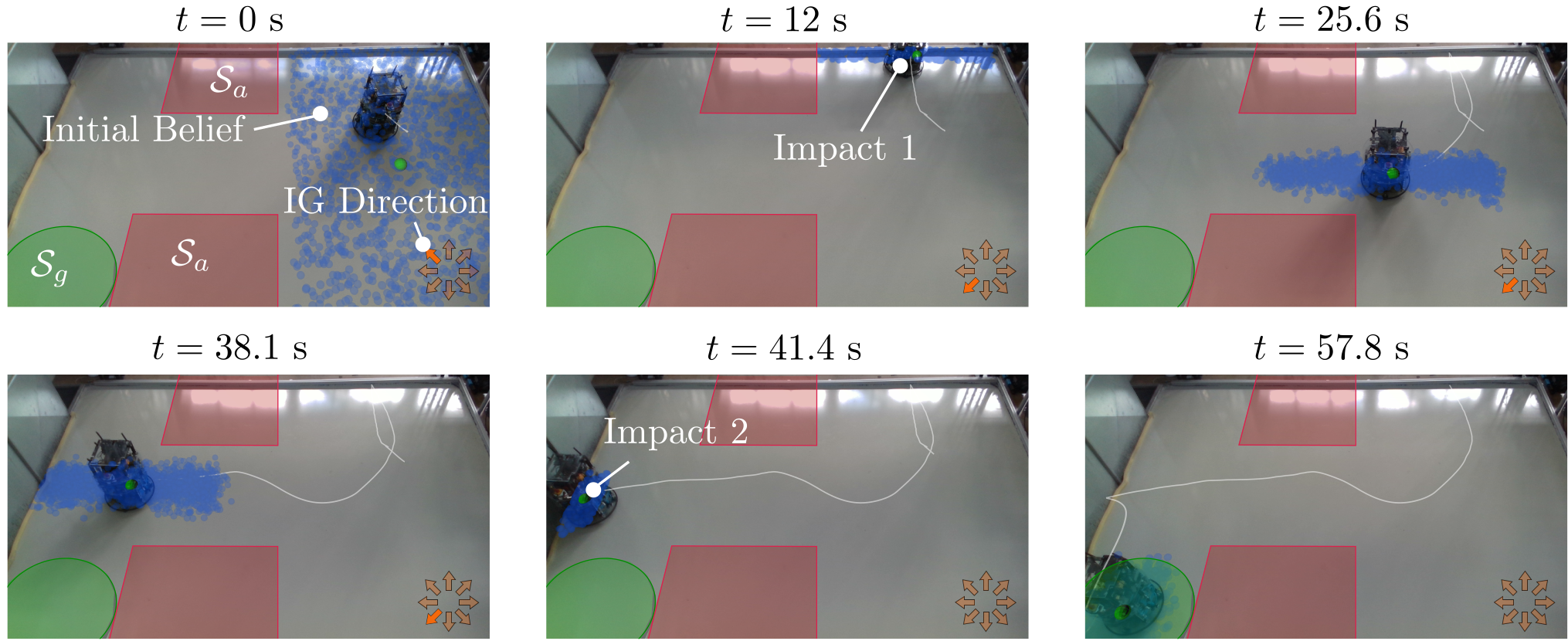}
    \caption{Hardware validation for a reach-avoid scenario with a narrow passage. To navigate to the goal region $\mathcal{S}_g$, the vertical uncertainty in the position ($t=0$s) needs to be reduced which is done with a bump with the top of the workspace ($t=12$s). The belief narrows sufficiently and passage with probabilistic safety guarantees can be facilitated ($t=25.6$s). In order to also provide the reach behavior with the desired probabilistic guarantee, a second bump ($t=41.4$s) narrows the belief in the horizontal direction.}
    \label{fig:SpaceExperiment}
\end{figure*}

The dynamics of the space platform are given as stochastic double integrator dynamics
\begin{align}
    \begin{bmatrix}
        \text{d}\bm{p}\\
        \text{d}\bm{v}
    \end{bmatrix} &= \begin{bmatrix}
        \bm{v}\\
        \bm{u}
    \end{bmatrix} \text{d}t + \bm{\sigma} \text{d}\bm{W}\\
    z &= \{0, 1\}
\end{align}
where $\bm{p}\in \mathbb{R}^2$ denotes the robot's position, $\bm{v}\in \mathbb{R}^2$ is its velocity and $\ub \in \mathbb{R}^2$ is an acceleration control input. 
Similar to the original Constrained Bumper Scenario, we assume we have access to the map and the goal and avoid regions.
% \textcolor{blue}{need to mention here that we know the map and that the robot is only allowed to bump into the left/right and uppper wall as it gets stuck on the lower one. Also, we didn't use the IMU for odometry estimates but noisy MOCAP. Should write that as well.}
Constrained by the quality of the epoxy-coated floor, we only allow uncertainty reduction via impacts with the left, right and upper wall, as the air-bearings get stuck near other parts of the border.
We utilize onboard IMU measurements for detecting impacts with the border and artificially noise-corrupted motion-capture data to estimate the robot’s velocity. We learn the BCLF purely in simulation and directly transfer it to hardware.
We solve the information-gathering control at $10$Hz and the BCBF safety filter at $50$Hz with the belief represented by $8000$ particles.

The first scenario, shown in Fig.~\ref{fig:SpaceExperiment}, involves a repair station $\mathcal{S}_g$ that can only be reached through a narrow corridor. From an initial uniform particle belief, the space platform reduces its uncertainty via a bump at the top of the workspace, after which the distribution is narrow enough to pass the corridor. A second bump after passage ensures the distribution is narrow enough to guarantee the reach-behavior with the desired probability.

A second reach-avoid scenario, shown in Fig.~\ref{fig:FirstPage}, highlights how different initial conditions (under the same initial belief $\bb^0$) lead to qualitatively different trajectories. The initial location $\xb_0^1$ is far enough to the right in the initial belief that only a single bump is required with the top of the workspace. Conversely, $\xb_0^2$ experiences an additional bump with the left side of the workspace after which the uncertainty is still too significant in the vertical direction to enable the avoid and reach behavior. This leads to a second bump at the top of the workspace after which the uncertainty is low enough to reach the goal.

Note that between these scenarios, the BCLF did not need to be retrained as its information-gathering behavior is only dependent on the shape of the workspace and is robust to different initial beliefs.
The scenarios presented here highlight the applicability of the proposed control architecture to real-world robotics challenges. Thanks due to efficient QP formulations in Eq.~\ref{eq:LyapControl} and Algorithm~\ref{alg:conformal_qp}, a real-time control rate is achievable with a large number of particles leading to safe deployment under partial observability.
% \textcolor{red}{Do we actually want to say anything more because we don't quantitatively asses the safety for example. So how robust are we really?}

% A second scenario, could be a simple end-effector that has a camera attached and tries to find a desired object in an occluded space. Here we could also train on a variety of environments, e.g. environments where the occluded area changes.

\section{CONCLUSIONS AND FUTURE WORK}
This paper presented a layered, certificate-based control architecture for reach-avoid POMDPs that operates directly in belief space. By decoupling goal reaching, information gathering, and safety into separate modules, we showed how conflicting objectives under partial observability can be addressed in a principled and scalable manner. Central to this architecture are Belief Control Lyapunov Functions (BCLFs) for information gathering and Belief Control Barrier Functions (BCBFs) for risk-aware safety, each providing their own probabilistic guarantees and operating at appropriate time scales.

We formalized information gathering as a Lyapunov convergence problem in belief space and established theoretical conditions under which reinforcement learning value functions induce valid stochastic and finite-time BCLFs. For safety, we extended belief-space CBFs with conformal prediction to obtain finite-horizon probabilistic safety guarantees, moving beyond pointwise-in-time assurances. The resulting control synthesis reduces to lightweight optimization problems that can be solved in real time even for high-dimensional, non-Gaussian belief representations.

Through extensive simulation studies and hardware experiments on a space robotics-platform, we demonstrated that the proposed architecture improves both safety and task performance compared to state-of-the-art constrained POMDP solvers. In particular, the modular design enables effective information gathering, resolves conflicts between safety and exploration, and supports reuse of learned BCLFs across different tasks without retraining.

Future work will focus on improving scalability to higher-dimensional state spaces by exploring belief representations beyond standard particle filters, such as Stein variational particle filters~\cite{maken2022stein}. Another promising direction is to leverage the reasoning capabilities of foundation models to maintain a lower-dimensional belief over the environment. Such compact belief representations could then be used to learn BCLFs that enable information gathering in complex, visually rich environments.

In addition, it would be valuable to further investigate the probabilistic guarantees of the overall control architecture and the interactions between the different control modules. Since the BCLFs considered in this work are based on action-value functions with discretized action spaces, an important extension is to develop BCLFs for continuous action spaces, for example by leveraging algorithms such as TD3~\cite{fujimoto2018addressing} that learn continuous Q-functions. Finally, we plan to study formal and probabilistic verification of learned BCLFs, for instance by incorporating an additional layer of conformal prediction as proposed in~\cite{lin2024verification}.

\balance
\bibliographystyle{IEEEtran}
\bibliography{references}

\section{Appendix}
\subsection{Proof of Theorem \ref{thm:BCLF_value}}
\label{sec:proof_BCLF}
\begin{proof}
    First, we show that the CLF candidate is greater than zero everywhere outside the goal set and zero inside. 
    If the state is inside the goal set, the reward is always $r(\xb, \ub) = 0$, hence 
    \begin{align}
        W\ofx = -V^*\ofx = -\sum_{k=0}^{\infty} \gamma^k r(\xb_k, \ub_k) = 0.
    \end{align}
    If the state is outside the goal set, and we have strictly negative rewards for all $\xb \notin \mathcal{S}_g$, we obtain $W\ofx = -V^*\ofx > 0$, thus the first conditions of a stochastic CLF in Def \ref{def:SCLF} hold.

    Next, we show the expected decrease condition. 
    From the Bellman equation, we have
    \begin{align}
        Q^*(\xb,\ub) = \mathbb{E}_w\!\left\{\, r(\xb,\ub) + \gamma V^*\!\left(\bm{F}(\xb,\ub,\bm{w})\right) \right\}.
    \end{align}
    Substituting $W\ofx = -V^*\ofx$ yields
    \begin{align}
        \mathbb{E}_w\!\left\{\, W\!\left(\bm{F}(\xb,\ub,\bm{w})\right) \right\} 
        = \frac{1}{\gamma}\!\left( r(\xb,\ub) - Q^*(\xb,\ub) \right).
    \end{align}
    For the optimal control input $\ub^*(\xb)$, we have $Q^*(\xb,\ub^*(\xb)) = V^*\ofx = -W\ofx$, which gives
    \begin{align}
        \mathbb{E}_w\!\left\{\, W\!\left(\bm{F}(\xb,\ub^*(\xb),\bm{w})\right) \right\}
        \le \frac{1}{\gamma}\!\left( R_{\max} + W\ofx \right).
    \end{align}
    For all $\xb \in \mathcal{D}$, we have $W\ofx \le W_{\max}$, so that
    \begin{align}
        \mathbb{E}_w\!\left\{\, W\!\left(\bm{F}(\xb,\ub^*(\xb),\bm{w})\right) \right\}
        &\le \frac{1}{\gamma}\!\left( 1 + \frac{R_{\max}}{W_{\max}} \right) W\ofx 
        \\
        &= c_{\min}\, W\ofx.
    \end{align}
    Since $R_{\max} < 0$ and $W_{\max} < |R_{\max}| / (1 - \gamma)$, it follows that $c_{\min} < 1$. 
    Hence, there exists a constant $c \in [c_{\min}, 1)$ such that
    \begin{align}
        \mathbb{E}_w\!\left\{\, W\!\left(\bm{F}(\xb,\ub,\bm{w})\right) \right\} \le c\, W\ofx,
    \end{align}
    for some control input $\ub \in \mathcal{U}$ (in particular, $\ub^*(\xb)$).
    Since $W\ofx > 0$ for all $\xb \notin \mathcal{S}_g$ and $W\ofx = 0$ for $\xb \in \mathcal{S}_g$, 
    all SCLF conditions in Def. \ref{def:SCLF} hold.
    Therefore, $W\ofx = -V^*\ofx$ is a stochastic Control Lyapunov Function on $\mathcal{D}$.
\end{proof}

\subsection{Proof of Theorem \ref{thm:FBCLF_value}}
\label{sec:proof_FBCLF}
\begin{proof}
    By the same logic of the proof in Sec. \ref{sec:proof_BCLF}, we know that $W\ofx \leq 0~\forall \xb \in \mathcal{S}_g$ and $W\ofx >0~\forall \xb \notin \mathcal{S}_g$. Further, we can obtain
    \begin{align}
        \Delta W &:= \mathbb{E}_w\!\left\{\, W\!\left(\bm{F}(\xb,\ub^*(\xb),\bm{w})\right) \right\} - W\ofx\\
        &\leq \frac{1}{\gamma}\!\left( R_{\max} + (1-\gamma)W\ofx \right)=: \phi(W).
    \end{align}
    Note that $\phi(W)$ is affine and strictly increasing in $W$ since $1-\gamma>0$. For fixed-time convergence, we need to show that $\Delta W \leq -\min(W\ofx, \eta)$. First, we consider the region $0< W\ofx \leq \eta$ for which we require $\Delta W \leq - W$. As $\phi(W)$ is strictly increasing, it suffices to check at the worst case $W\ofx = \eta$, i.e. $\phi(\eta)\leq - \eta$:
    \begin{align}
        \frac{(1-\gamma)\eta + R_{\max}}{\gamma} &\leq -\eta\\
        \Leftrightarrow (1-\gamma)\eta + R_{\max} &\leq -\gamma \eta\\
        \Leftrightarrow\eta + R_{\max} &\leq 0
    \end{align}
    which holds by assumption. Next, consider the region $\eta \leq W\ofx \leq W_{\max}$ for which it, again, suffices to check $\phi(W_{\max}) \leq -\eta$:
    \begin{align}
        \frac{(1-\gamma)W_{\max} + R_{\max}}{\gamma} &\leq -\eta\\
        \Leftrightarrow \frac{(|R_{\max}| -\gamma\eta) + R_{\max}}{\gamma} &\leq -\eta\\
        \Leftrightarrow-\eta &\leq -\eta
    \end{align}
    which concludes that $W\ofx=-V^*\ofx$ is a local finite-time stochastic CLF on the domain $\mathcal{D}$. 
\end{proof}
\subsection{Proof of Proposition \ref{prop1}}
\label{sec:proofprop}
The proof is adapted from~\cite{yang2023safe}.
\begin{proof}
    We start with
    \begin{align}
        &~~~~~\mathrm{Pr} [\xb_t \notin \mathcal{S}_a, \forall t \in \cup_{k=1}^M I_k]\\
        &=\mathrm{Pr} [\xb_t \notin \mathcal{S}_a, \forall t \in I_M \mid \xb_t \notin \mathcal{S}_a, \forall t \in \cup_{k=1}^{M-1} I_k]\\
        &\cdot \mathrm{Pr} [\xb_t \notin \mathcal{S}_a, \forall t \in \cup_{k=1}^{M-1} I_k]\\
        &= \mathrm{Pr} [\xb_t \notin \mathcal{S}_a, \forall t \in I_M \mid \xb_{t_{M-1}} \notin \mathcal{S}_a]\\
        &\cdot \mathrm{Pr} [\xb_t \notin \mathcal{S}_a, \forall t \in \cup_{k=1}^{M-1} I_k]\\
        &= (1 - \bar{\delta}_a) \cdot \mathrm{Pr} [\xb_t \notin \mathcal{S}_a, \forall t \in \cup_{k=1}^{M-1} I_k]
    \end{align}
    By Assumption~\ref{ass:iid}, we assume that a discrete transition on the belief state does not leave the safe set, thus we can apply this result recursively, to obtain
    \begin{align}
        \mathrm{Pr} [\xb_t \notin \mathcal{S}_a, \forall t \in \cup_{k=1}^M I_k] &= (1 - \bar{\delta}_a)^M \cdot \mathrm{Pr}[\xb_0 \notin \mathcal{S}_a]\\
        &= (1 - \bar{\delta}_a)^M.
    \end{align}
\end{proof}

\end{document}